\documentclass[onecolumn, journal]{IEEEtran}

\newcounter{algoline}

\usepackage[autolanguage,np]{numprint} 

\usepackage{arydshln}
\usepackage{booktabs}
\usepackage[binary-units=true]{siunitx}
\usepackage{times}
\usepackage{graphicx}
\usepackage{latexsym}
\usepackage{pifont}
\newcommand{\cmark}{\ding{51}}
\newcommand{\xmark}{\ding{55}}
\usepackage{xfrac}

\usepackage[utf8]{inputenc}
\usepackage{tikz}
\usetikzlibrary{arrows}
\tikzset{
	treenode/.style = {align=center, inner sep=0pt, text centered, font=\sffamily},
	arn_n/.style = {treenode, circle, black, font=\sffamily\bfseries, draw=black, fill=white, text width=1.5em},
	arn_r/.style = {treenode, circle, red, draw=red, text width=1.5em, very thick},
	arn_x/.style = {treenode, rectangle, draw=black, minimum width=0.5em, minimum height=0.5em}
}

\usetikzlibrary{fit}
\usetikzlibrary{shapes.geometric}
\usetikzlibrary{trees}
\usepackage{verbatim}
\usepackage{amsthm}
\usepackage{amssymb}
\newtheorem{definition}{Definition}
\usepackage[onelanguage, ruled, vlined]{algorithm2e}
\usepackage{pgfplots}
\pgfplotsset{compat=newest,
	colormap = {whiteblack}{color(0cm)  = (white);color(1cm) = (black)}
}
\newtheorem{example}{Example}
\newtheorem{theorem}{Theorem}

\usepackage{float}
\usepackage{multicol}
\usepackage{lipsum}
\usetikzlibrary{arrows.meta}	
\usepackage{mathtools}
\usepackage{blkarray, bigstrut}

\usepackage{makecell}	
\usepackage{rotate}
\usepackage{mathrsfs}

\newcommand{\vars}{\mathscr X}
\newcommand{\cons}{\mathscr C}
\newcommand{\agents}{\mathscr A}
\newcommand{\dom}{\mathscr D}
\newcommand{\Util}{\mathscr U}
\newcommand{\Rew}{\mathscr R}

\newcommand{\as}[1]{\ensuremath{(#1)}}
\newcommand{\vas}[1]{\ensuremath{[#1]}}

\newcommand{\probabilityD}{\ensuremath{probD}}
\newcommand{\agentview}{{\tt agentView}}
\newcommand{\okm}{{\tt ok?}}
\newcommand{\addlinkm}{{\tt addlink}}
\newcommand{\nogoodm}{{\tt nogood}}

\newcommand{\AgentOne}{{\ensuremath{A_1}}}
\newcommand{\AgentTwo}{{\ensuremath{A_2}}}
\newcommand{\AgentThree}{{\ensuremath{A_3}}}
\newcommand{\Students}{Students~{\ensuremath{A_1}},
	{\ensuremath{A_2}}, and {\ensuremath{A_3}}}

\newcommand{\estcost}{$estimatedCost$\xspace}
\newcommand{\risk}{agreementProb}

\newcommand{\solutionCost}{$SolutionCost$}
\newcommand{\privacyCost}{$PrivacyCost$}

\newcommand{\SyncBT}{{\textsl{SyncBT}}\xspace}   
\newcommand{\ABT}{{\textsl{ABT}}\xspace} 
\newcommand{\DBO}{{\textsl{DBO}}\xspace} 
\newcommand{\DSA}{{\textsl{DSA}}\xspace} 
\newcommand{\ADOPT}{{\textsl{ADOPT}}\xspace} 
\newcommand{\SyncBTU}{{\textsl{SyncBTU}}\xspace} 
\newcommand{\ABTU}{{\textsl{ABTU}}\xspace} 
\newcommand{\DBOU}{{\textsl{DBOU}}\xspace} 
\newcommand{\DSAU}{{\textsl{DSAU}}\xspace} 
\newcommand{\ADOPTU}{{\textsl{ADOPTU}}\xspace} 	
\newcommand{\DisPrivCSP}{{\textsl{DisPrivCSP}}\xspace }
\newcommand{\VPS}{{\textsl{VPS}}\xspace }
\newcommand{\PKC}{{\textsl{PKC}}\xspace }
\newcommand{\DPOP}{{\textsl{DPOP}}\xspace }
\newcommand{\SSDPOP}{{\textsl{SS-DPOP}}\xspace }
\newcommand{\SyncBB}{{\textsl{SyncBB}}\xspace }
\newcommand{\PSyncBB}{{\textsl{P-SyncBB}}\xspace }
\newcommand{\NCBB}{{\textsl{NCBB}}\xspace }
\newcommand{\AFB}{{\textsl{AFB}}\xspace }

\newcommand{\DisCSP}{{\textit{DisCSP}}\xspace }
\newcommand{\DisCSPs}{{\textit{DisCSPs}}\xspace }
\newcommand{\DCOP}{{\textit{DCOP}}\xspace }
\newcommand{\DCOPs}{{\textit{DCOPs}}\xspace }
\newcommand{\MOOP}{{\textit{MOOP}}\xspace }
\newcommand{\MODCOP}{{\textit{MO-DCOP}}\xspace }
\newcommand{\MODCOPs}{{\textit{MO-DCOPs}}\xspace }
\newcommand{\UDisCSP}{{\textit{UDisCSP}}\xspace }
\newcommand{\UDisCSPs}{{\textit{UDisCSPs}}\xspace }
\newcommand{\UDCOP}{{\textit{UDCOP}}\xspace }
\newcommand{\UDCOPs}{{\textit{UDCOPs}}\xspace }
\newcommand{\DisCP}{{\textit{DisCP}}\xspace }
\newcommand{\DisCPs}{{\textit{DisCPs}}\xspace }
\newcommand{\UDisCP}{{\textit{UDisCP}}\xspace }
\newcommand{\UDisCPs}{{\textit{UDisCPs}}\xspace }
\newcommand{\POMDP}{{\textit{POMDP}}\xspace }
\newcommand{\POMDPs}{{\textit{POMDPs}}\xspace }

\newcommand{\estimateCostDisCSP}{{\tt estimateCostDisCSP}\xspace }

\newcommand{\backtrack}{{\tt backtrack}\xspace }
\newcommand{\interruptSolving}{{\tt interruptSolving}\xspace }
\newcommand{\checkagentviewABTU}{{\tt checkAgentView\_ABTU}\xspace }

\newcommand{\assignCPA}{{\tt assignCPA}\xspace }
\newcommand{\estimateCostDCOP}{{\tt estimateCostDCOP}\xspace }
\newcommand{\checkAgentViewADOPTU}{{\tt checkAgentView\_ADOPTU}\xspace }
\newcommand{\sendImproveDBOU}{{\tt sendImproveDBOU}\xspace }
\newcommand{\procDSAU}{{\tt DSAU}\xspace }

\newcommand{\costTerminal}{\ensuremath{costRound}}
\newcommand{\CostTemp}{\ensuremath{costTemp}}

\def\Aa{Professor}
\def\Aaa{{Professor \ensuremath{A_1}}}
\def\Aab{{Student \ensuremath{A_2}}}
\def\Aac{{Student \ensuremath{A_3}}}
\def\Ab{Student A2}
\def\Ac{Student A3}
\def\M{M}

\newcommand{\eg}{{\textit{e.g.},}\xspace}
\newcommand{\ie}{{\textit{i.e.},}\xspace}

\begin{document}
	
\title{Distributed Constraint Problems for Utilitarian Agents with Privacy Concerns, Recast as POMDPs}

\author{Julien Savaux, Julien Vion, Sylvain Piechowiak, Ren\'{e} Mandiau,~\IEEEmembership{LAMIH UMR CNRS 8201, University of Valenciennes, France}\\
	Toshihiro Matsui,~\IEEEmembership{Nagoya Institute of Technlogy, Japan}\\
	Katsutoshi Hirayama,~\IEEEmembership{Kobe University, Japan}\\
	Makoto Yokoo,~\IEEEmembership{Kyushu University, Japan}\\	
	Shakre Elmane, Marius Silaghi,~\IEEEmembership{Florida Institute of Technology, USA}\\
	}
		
\maketitle
	
	\begin{abstract}
		
		Privacy has  traditionally been a major motivation for distributed problem solving. 
		Distributed Constraint Satisfaction Problem (\DisCSP) as well as Distributed Constraint 
		Optimization Problem (\DCOP) are fundamental models used to solve various families 
		of distributed problems. Even though several approaches have been proposed to quantify 
		and preserve privacy in such problems, none of them is exempt from limitations.
		Here we approach the problem by assuming that computation is performed among 
		utilitarian agents. We introduce a utilitarian approach 
		where the utility of each state is estimated as the difference between the reward for 
		reaching an agreement on assignments of shared variables and the cost of privacy loss. 
		We investigate extensions to solvers where agents integrate the utility function to 
		guide their search and decide which action to perform, defining thereby their policy. 
		We show that these extended solvers succeed in significantly 
		reducing privacy loss without significant degradation of the solution quality.		
	\end{abstract}
	
\IEEEpeerreviewmaketitle	
	\LinesNumbered

	\section{Introduction}\label{Introduction}
	
	Privacy is an important problem in a lot of distributed applications,
	therefore, the reward for solving the problem should be considered, but also
	the cost of privacy loss during the process~\cite{greenstadt2006}. 
	For example, when users exchange information on
	social networks~\cite{krupa2012},
	they reveal 
	(often unconsciously)  personal data (\eg  age,  address, date of birth). 
	Related studies were also performed in the domain of Ambient Intelligence which is concerned with the distributed management of confidential information  between different components (\eg camera, computer, PDA), to allow or prevent the sharing of resource information~\cite{piette2016}.
	
	In distributed scheduling problems, 
	confidentiality also happen when information is exchanged between agents/participants (for example, a participant
	does not want to reveal his/her unavailability for a time slot because the explanation concerns his/her
	private life, and it should not be necessarily discussed publicly).
	Indeed, we know that the assignment of time slots can be difficult if participants do not
	want to reveal their constraints~\cite{freuder2001,crepin2009}.  Such coordinated decisions are in conflict with the need to keep constraints
	private~\cite{faltings2008}.
	We 	consider that these coordinated decisions may be defined 
	by a set of different constraints distributed among different agents.
	
	In Distributed Constraint Satisfaction Problems (\DisCSPs) and Distributed Constraint Optimization Problems (\DCOPs), agents have to assign values to variables while 
	respecting given constraints.
	To find such assignments, agents exchange messages until a solution is found or until an agent detects that there is none.
	Thus, agents have to reveal information during search, 
	causing privacy to be a major concern in
 \DisCSPs~\cite{yokoo1998,HamadiBessiereQuinqueton98:HBC98}. 
		If an agent is
	concerned about its privacy, then it can associate a cost to the
	revelation of each information in its local problem. This cost may be embedded into utility driven reasoning.
	
	A common assumption is that utility-based agents associate each state with 
	a utility value~\cite{russell1995}.
	As such, each communication action's utility is evaluated 
	as the difference between initial and 
	final utilities. Indeed, if an agent is concerned about its privacy, then it can associate a cost to the revelation of each information in its local 
	problem. Since they are interested in solving the problem, they must also be able to quantify the reward they draw from finding the solution. Here we approach the 
	problem by assuming that privacy is a utility that can be aggregated with the reward for solving this problem. We evaluate how much privacy is lost by 
	the agents during the problem solving process, by the sum of the utility lost for each information that was revealed. For example, the existence and availability of a value from the domain of a variable 	is the kind of information that the agents want to keep private. The cost of a constraint for a solution is another	example of information that agents would like to keep private.
	
	While sometimes possibilistic reasoning is  used to guide search,
	agents were usually assumed to participate in the search process until an agreement  is found between all different agents (\ie a solution). 
	We investigate the case where an agent may modify its search process to optimize utility. Two extensions are introduced, Utilitarian Distributed Constraint Satisfaction Problem (\UDisCSP) and Utilitarian Distributed Constraint Optimization Problem (\UDCOP), addressing \DisCSPs and \DCOPs, respectively.
	These extensions exploit the rewards of agreements and costs representing privacy loss as guidance for the 	utility-based agents, where the utility of each 
	state is estimated as the difference between the expected rewards for 
	agreements, and the expected cost representing privacy loss.
	In this work, Distributed Constrained Problem (\DisCP) will refer to both \DisCSP and \DCOP. Similarly, Utilitarian Distributed Constrained 
	Problems (\UDisCPs) will refer to both \UDisCSPs and \UDCOPs.

	The paper is organized as follows.
	Section~\ref{Background} presents existing research concerning 
	solving algorithms and approaches to privacy for 
	\DisCPs.
	Further, Section~\ref{Contribution} describes the concepts and 
	solvers involved in \UDisCP
	to preserve privacy. Section~\ref{TheoreticalDiscussion} discusses theoretical implications. 
	Section~\ref{Experiments} reports  our experimental results and  evaluates  privacy loss on distributed meeting scheduling (DMS) problems.
	Section~\ref{Conclusions} presents our conclusions and our directions for future research.

	\section{Background}\label{Background}
	
	We present distributed constrained problems (\ref{DisCP}), and existing approaches for privacy as well as their limits (\ref{defprivacy}).
	
	\subsection{Distributed Constrained Problem} \label{DisCP}
	
	Distributed Constraint Satisfaction Problems (\DisCSPs) and 
	Distributed Constraint Optimization Problems (\DCOPs)
	have been extensively studied as a fundamental way of
	modelling constrained problems in
	multi-agent systems, and will be defined in the following, as well as existing solvers.
		
	\subsubsection{Definitions}
Let us first remind the definitions of the
Distributed Constraint Satisfaction Problem (\DisCSP) and 
of the
Distributed Constraint  Optimization Problem (\DCOP)~\cite{yokoo1998}. 

	\begin{definition}
	A Distributed Constraint Optimization Problem (\DCOP) is formally defined as a tuple $\langle{\agents,\vars,\dom,\cons}\rangle$ where:
	\begin{itemize}
		\item  $\agents = \{A_1, A_2, \dots, A_m\}$ is a finite set of $m$ agents. 
		\item $\vars = \{x_1,\linebreak[1] x_2, \dots,\linebreak[1]
		x_n\}$ is a finite set of $n$ variables. Each agent $A_{self}$ encapsulates variables denoted $\vars(A_{self})$
		(with $\vars(A_{self}) \subseteq \vars$).   
		\item $\dom=\langle \dom_1,\dom_2, \ldots, \dom_n\rangle$ is a set of $n$ domains. $\dom_{self}$ is the set of possible values for
        $\vars(A_{self})$. 
		\item  $\cons = \{C_1, C_2, \dots, C_e\}$ is a finite set of $e$ 
		valued 	constraints. Each
		constraint $C_i$ involves   some variables $\vars(C_i) \subseteq
		\vars$ 	
		defining a cost (positive value) for assignments.  We note that $\cons_{self}= \{ C_i \in \cons | \vars(C_i) \cap \vars(A_{self}) \ne \emptyset \}$.
	\end{itemize}
	The objective is to find an assignment for each variable  that	minimizes the total cost. 
	\end{definition}

\begin{definition}
	A Distributed Constraint Satisfaction Problem (\DisCSP) is defined as a \DCOP where the constraints are 
predicates, each one defining for sets of assignments a cost of 
$\infty$ (constraint violation).
	The  problem is to find an assignment to variables  that does not violate any constraint. 
\end{definition}

\begin{definition}
	Distributed Constrained  Problem (\DisCP) is any problem modelled with \DisCSP or one of its extensions. 
\end{definition}
	
	In the following, we study a particular case of problems, namely  mono-variable problems where $m=n$ (\ie a variable per agent),
	 $\vars(A_{self})$ containing a single variable, $x_{self}$.
	  Also, in the following frameworks and algorithms, $x_{self}$, $\dom_{self}$, and $\cons_{self}$
	 are generalized 
	 referring to the projection of the part of different  elements from the problem known only by the current agent (called $self$).

	The problem that each agent has to solve in \DisCP{} is a Stochastic Constraint Optimization Problem (Stochastic COP) or a Stochastic CSP~\cite{walsh2002stochastic}, which are generalizations of Stochastic SAT, SSAT~\cite{littman2001stochastic}, namely where the local problem has to be solved considering its impact on external variables and constraints whose values are not yet known and that are not under the control of this agent, but that of other agents (commonly, lower priority), values of which the agent may know a probabilistic profile.
	
	\subsubsection{Existing Solvers}

	We now introduce some existing solvers for  \DisCP 
	with which we exemplify extensions based on
	our utilitarian approach. We consider two well-known solvers for \DisCSP (\SyncBT and \ABT),
and three solvers for \DCOPs (called \ADOPT, \DSA and \DBO):
	
	\paragraph{Synchronous Backtracking (\SyncBT)}
	is the baseline algorithm for \DisCSPs~\cite{yokoo1992,zivan2003}.
	\SyncBT is a simple distribution of the standard backtracking algorithm.  
	Agents consecutively send a satisfying assignment to their variable to the 
	next agent. If an agent is unable to find an 
	instantiation compatible with the current partial assignment it has received, 
	it asks the previous one to change its assignment.
	The process repeats	until a complete solution is built or until the 
	whole search space is explored.
	
	\paragraph{Asynchronous Backtracking (\ABT)}
	is a common alternative solver for \DisCSPs that 
	allows agents to run concurrently 
	and asynchronously~\cite{yokoo1992}. 
	Each agent finds an assignment to 
	its variable and communicates it to the other connected agents
	having constraints involving this variable.
	Then agents wait for incoming
	messages. The assignments received through \okm{} form a context 
	called \agentview{}. If an agent's assignment is inconsistent with its 
	\agentview{}, it is changed and communicated to the other agents.
	A subset of an \agentview{} preventing an agent from finding an assignment that
	does not violate any of its constraints is called a nogood. 
	If an agent infers a nogood from its constraints and its \agentview{}, it
	asks the lowest priority agent involved in the nogood to change 
	its assignment through a \nogoodm{} message.
		
	\paragraph{Asynchronous Distributed Optimization (\ADOPT)}
	guarantees to find the optimal solution and only needs polynomial space ~\cite{modi2005,yeoh2010,Silaghi2009}.	
	\ADOPT organizes agents into a depth first search tree in which 
	constraints are only 
	allowed between a variable and any of its acquaintances 
	(parents and descendants in the tree).
	When the estimated cost of agents' assignment is higher than a given threshold, 
	the agent switches its value assignment to the value 
	with smallest estimated cost.
	When the upper and lower bounds meet at the root agent, a globally optimal solution has been found and the process finishes. Note that other approaches based on a graph re-arrangement have been explored, for example a cluster exploitation like Asynchronous Partial Overlay~\cite{mailler2006,Mailler2012}.
	
	\paragraph{Distributed Stochastic Algorithm (\DSA)}
	makes agents start their process by randomly selecting a value~\cite{zhang2002}.  
	Agents then enter an iteration where they send their last assigned value to
	their neighbours and collect any new values from them.
	The next candidate value is chosen based on the values
	received from other agents and on maximizing a given utility function. 
	This algorithm is incomplete and does not guarantee optimality, but it is frequently efficient
	for finding solutions close to optimum.
	
	\paragraph{Distributed Breakout (\DBO)}
	is an iterative improvement solver for \DCOPs~\cite{yokoo1996,Hirayama2005}. 
	The evaluation of a given
	solution is the summation of the weights for all its violated constraints. 
	An assignment is then changed to decrease the solution value.	
	If the evaluation of the solution cannot be decreased by changing a value,
	the current state may be a local minimum. In this situation, 
	\DBO increases the weights of constraint violation 
	pairs in the current state so
	that the evaluation of this current state becomes higher than the
	neighbouring states. Thus, the algorithm can escape from the local
	minimum.

	\subsection{Privacy} \label{defprivacy}
	
	We present generalities about privacy, in particular about 
	our definition and typology. This typology will be used to compare 
	existing approaches to privacy in \DisCP.
	
	\subsubsection{Definition and typology}
	
	Privacy is the concern of agents to not reveal their personal information. In this work, we define privacy as follows:
	\begin{definition}
		Privacy is the utility that agents benefit from conserving the secrecy of their personal information.
	\end{definition} 
	
	Contrary to the standard rewards in \DisCSPs, 
	privacy costs are proper to each individual 
	agent. Therefore, the computation is now performed 
	by utility-based and self-interested agents,
	whose decisions aims at maximizing a utility function. 
	The objective is then to define a policy
   associating an expected utility maximizing action (communication act or computation) to each state, where the state includes the belief about the global state).
	In existing works, several approaches have been developed to deal with
	privacy in \DCOPs. 
	Some cryptographic approaches offer certain end-to-end security guarantees by integrating the entire solving process in one primitive for 
	\DisCSPs or for \DCOPs, the highest level of such privacy guarantees being achievable only for problems with a single 
	variable~\cite{Silaghi2004}. 
	Other cryptographic approaches are hybrids 
	interlacing cryptographic and artificial intelligence steps~\cite{Tassa2015,greenstadt2007,grinshpoun2014}. While ensuring
	privacy in each primitive~\cite{hirt2000}, cryptographic techniques are usually
	slower, and sometimes require the use of external servers or computationally
	intensive secure function evaluation techniques that may not always be
	available or justifiable for their benefits, making them impractical~\cite{greenstadt2006}, or lacking clear global 
	security guarantees.
	A couple of such approaches with which we compare in more detail are:

	\paragraph{Distributed Pseudo-tree Optimization Procedure with Secret Sharing (\SSDPOP)} 
	modifies the standard
	\DPOP algorithm~\cite{petcu2005} to protect leaves in the depth first search tree, where agents sharing constraints 
	are on the same branch~\cite{greenstadt2007}. \SSDPOP uses
	secret sharing~\cite{shamir1979} to aggregate the results of a
	single solution, without revealing the individual valuations this solution consists in. The aggregated values are then
	passed to the bottom agent, who aggregates this information with its
	own valuations and sends the aggregate up the chain. 
	\cite{LeauteF13} has also extended \DPOP to 
	preserve different types of privacy using secure multi-party computation.

	\paragraph{Privacy-Preserving Synchronous Branch and Bound (\PSyncBB)}
	is a cryptographic 	version
	of \SyncBB~\cite{hirayama1997,Grinshpoun2016} for solving \DCOPs 
	while 
	enforcing a stronger degree of 
	constraint privacy~\cite{grinshpoun2014}. \PSyncBB
 computes the costs of CPAs 
	(current partial assignments) and compares them to the 
	current upper bound, using secure multi-party protocols.
	Some protocols were proposed in~\cite{grinshpoun2014} 
	that can solve problems
	 without resorting to 
	costly transfer sub-protocols, and compare
	the cost of a CPA  shared between two agents to 
	the upper bound held by only one of them.
	Some variations on the standard \SyncBB  
	include \NCBB and \AFB~\cite{gershman2009,grinshpoun2014}. 
		
	We choose to deal with privacy 
	by embedding it into agents' decision-making. 
	Other approaches use different metrics and frameworks to quantify privacy loss. 
	According to~\cite{grinshpoun2012}, agents privacy  may concern the four following aspects:
	
	\begin{description}
		\item Domain privacy: Agents want to keep the domain of 
		their variable private. The common  benchmarks and some algorithms (like \DPOP, and common cryptographic techniques) assume 
		that all the domains are public, which leads to a complete loss 
		of domain privacy. In the original \DisCSP approach
		a form of privacy of domains is implicit (see \ABT), 
		while being formally required in its \PKC extension~\cite{brito2009}.
		\item Constraint privacy: Agents want to keep 
		the information related to their constraints private~\cite{silaghi2000}.
       If variables involved in constraints are considered to belong to 
       only one agent, we can distinguish the revelation of information 
		to agents that participate in the constraint (internal constraint privacy)
		and the one to other agents (external constraint privacy). 	
		While common problems with domain privacy can be straightforwardly modelled as problems with constraint privacy, as discussed later, there theoretically exists a kind of domain privacy which cannot be modelled with constraint privacy.
		\item Assignment privacy: Agents want to keep the 
		assigned values to their variables private.
		The revelation of assigned values concerns the 
		assignment of the final solution, as well as the ones
		proposed during search. 
		\item Algorithmic privacy: Even though it is 
		commonly assumed that all agents run the same algorithm during the solving,
		agents may modify the value of some parameters 
		guiding the search process for some personal benefit (\eg{}
		the likelihood of updating its value). This can be achieved by  keeping the 
		message structure 
		and contracts of certain existing \DisCSP solvers to be used as communication {\em protocols} rather than {\em algorithms}, as introduced in~\cite{silaghi2002comparison},
 where
		protocols obtained in such ways 
		are compared
		with respect to the flexibility offered for agents to hide their secrets.
	\end{description}	
	
	\subsubsection{Existing approaches to privacy in \DisCP}

	We now introduce two examples modelled with \DisCP frameworks, and show how
	existing approaches to privacy deal with these problems and their limits.
	
	\paragraph{Distributed Constraint Satisfaction Problem}
	
	\begin{example}  \label{ex:discsp}
		Suppose a meeting scheduling problem between a 
		professor (called \AgentOne{}) and two 
		students (\AgentTwo{} and \AgentThree{}). 
		They all consider to agree on a time slot to meet on a given day,
		having to choose between $8~am$, $10~am$ and $2~pm$. Professor 
		\AgentOne{} is unavailable
		at $2~pm$, \AgentTwo{} is unavailable at $10~am$, and \AgentThree{} 
		is unavailable at $8~am$.
		
		There can exist various reasons for privacy. For example,
		\AgentTwo{} does not want to reveal the fact that
		it is busy at $10~am$. 
		The value that 
		\AgentTwo{} 
		associates with not revealing 
		the $10~am$ unavailability is the salary from a second job ($\$ \numprint{2000}$). 
		The utility of finding an agreement for each student 
		is the stipend for their studies ($\$\numprint{5000}$). 
		 For \AgentOne{}, 
		the utility is  a fraction of the value of 
		its project ($\$\numprint{4000}$).
		This is an example of privacy for absent values or constraint tuples.
		
		Further \AgentThree{} had recently boasted to \AgentTwo{} that at
		$8~am$ it interviews for a job, and it would rather
		pay $\$\numprint{1000}$ than to reveal that it is not.
		This is an example of privacy for feasible values of constraint tuples.
	\end{example}

		Similarly, participants associate a cost to the 
		revelation of each availability and unavailability. 		
		Thus, scaling numbers by 1000 for simplicity, corresponding agents associate a cost 
		of $1$ to the revelation of their availability at 
		$8~am$, a cost of $2$ to the one at $10~am$, 
		and a cost of $4$ to the one at $2~pm$. 
		The reward from finding a solution is $4$ for \AgentOne{} and $5$ for both \AgentTwo{} and \AgentThree{}.
	
	For simplicity, in the next sections, we will refer to
	the possible values by their identifier: $1$, $2$, and $3$ (corresponding to $8~am$, $10~am$ and  $2~pm$ respectively).
	As this problem states allowed or forbidden values, it is represented by a \DisCSP as follows:

\begin{itemize}
	\item $\agents=\{\AgentOne{}, \AgentTwo{}, \AgentThree{} \}$
	\item $\vars=\{x_1,x_2,x_3\}$
	\item $\dom=\{ \{1, 2, 3\}, \{1, 2, 3\}$, $\{1, 2, 3\} \}$
	\item $\cons= \{ 
	\{\neg(x_1=x_2=x_3)\}, \{(x_1\not=3)\}, \{(x_2\not=2)\}, \{(x_3\not=1)\} \}$
\end{itemize}
	
	As it can be observed, \DisCSPs cannot model the information 
	concerning privacy. Now we will show how existing extensions model them.
	
	\paragraph{Distributed Private Constraint Satisfaction Problem
		(\DisPrivCSP)} models the privacy loss for individual 
	revelations~\cite{freuder2001,silaghi2002comparison}. It also lets
	agents abandon the search process when
	the incremental privacy loss overcomes the expected gains from finding
	a solution. Each agent pays a cost if the feasibility of each solution
	is determined by other agents. The reward for solving the problem is given
	as a constant. Those concepts were so far used for evaluating qualitatively existing
	algorithms, but were not integrated as heuristics in the search
	process. Privacy and the usual optimization criteria of Distributed Constraint
	Optimization Problems are merged into a unique criterion~\cite{doshi2008}.
	The additional parameters are a set of privacy coefficients 
	and a set of rewards.

	This framework successfully models all the information
	described in the initial problem and also measures the privacy loss
	for each agent. However, it was not yet investigated what is the
	impact of the interruptions when privacy loss exceeds the reward
	threshold, its relation to utility, or how agents could use this information to modify their
	behaviour during the search process to preserve more privacy.

	\paragraph{Valuation of Possible States (\VPS)}
	measures privacy loss by the extent to which the possible states of
	other agents are reduced~\cite{wallace2002,maheswaran2005,maheswaran2006,greenstadt2006}.
	Privacy is interpreted as a valuation on the other agents'
	estimates about the possible states that one lives in. 
	During the search process, agents propose their values in an order of
	decreasing preference. At the end of the search process, the
	difference between the presupposed order of preferences and the real
	one observed during search determines the privacy loss: the
	greater the difference, the more privacy has been lost.	
	 
	However in our sample problem,
	agents initially know nothing about others agents
	but the variable they share a constraint with and cannot 
	suppose an order of preference. 
	Agents have no information about others agents 
	privacy requirements.
	Thus, agents do not expect to receive any value proposal 
	more than another.
	In this direction one needs to extend \VPS to be able to also
	model the kind of privacy introduced in this example.

	\paragraph{Partially Known Constraints (\PKC)}
	uses entropy, as defined in information theory, to quantify 
	privacy loss~\cite{brito2009}. 
	In this method, two variables owned by two
	different agents may share a constraint. However, not all the forbidden
	couples of values involved in a constraint are known by both agents. Each agent only knows
	a subset of the constraints. 
	During the search process, assignment privacy is leaked
	through \okm{} and \nogoodm{} messages, like in standard
	algorithms. This problem is solved by not sending the value that is
	assigned to a variable in a \okm{} message, but the set of values compatible
	with this assignment. For \nogoodm{} messages, 
	rather than sending the current assignments,
	an identifier is used to specify the state of the resolution and to
	check if some assignments are obsolete or not.
	
	\PKC assumes that agents only know their 
	own individual unavailabilities~\cite{brito2009}. 
	Only the junction of information known by all agents can rebuild the whole problem.
	However, while \PKC let agents preserve privacy of unary constraints, it does not 
	consider the cost of revelation of assignments.
	
	\paragraph{Distributed Constraint Optimization Problem}
	
	\begin{example} \label{ex:dcop}			
		Suppose a problem concerning scheduling a meeting between three
		participants.  They all consider to agree on a place to meet on a same slot
		time, to choose between London, Madrid and Rome. For simplicity, 
		we will refer to these possible values by their
		identifiers: $1$, $2$ and $3$ (London, Madrid and Rome respectively).  
		
		\AgentOne{} lives in Paris, and it
		will cost it $\$70$, $\$230$ and $\$270$ to attend the meeting in
		London, Madrid and Rome respectively.  \AgentTwo{} lives in
		Berlin, and it will cost it $\$120$, $\$400$ and $\$190$ to attend
		the meeting in London, Madrid and Rome respectively. 
		\AgentThree{} lives in Brussels, and it will cost it $\$40$,
		$\$280$ and $\$230$ to attend the meeting in London, Madrid and Rome
		respectively. The objective is to find the meeting 
		location that minimizes the total cost
		students have to pay in order to attend.  
		
		The privacy costs for revealing
		its cost for locations $1$, $2$ and $3$ for \AgentOne{} are $\$80,\$20,\$40$.
		The privacy cost for locations $1$, $2$ and $3$ are $\$100,\$30,\$10$
		for \AgentTwo{} and $\$80,\$30,\$10$ for \AgentThree{}.  There exist various reasons	for privacy. For example, students may want to keep their cost for
		each location private, since it can be used to infer their initial
		location, and they would pay an additional (privacy) price rather 
		than revealing the	said travel cost. For example, \AgentOne{}
		associates $\$80$ privacy cost to the revelation
		of the travel cost of $\$70$ for meeting in London.
	\end{example}	
	
	This example may be defined by a \DCOP:
	
	\begin{itemize}
		\item
		$\agents=\{A_1, A_2, A_3 \}$
		\item
		$\vars=\{x_1,x_2,x_3\}$
		\item
		$\dom=$ $\{ \{1, 2, 3\}, \{1, 2, 3\}, \{1, 2, 3\} \}$
		\item
		$\cons= \{ 
		\{\as{x_1=1}, 70\}, \{\as{x_1=2}, 230\}, \{\as{x_1=3}, 270\},$ 
		\\
		\hspace*{1.0cm}$ 
		\{\as{x_2=1}, 120\}, \{\as{x_2=2}, 400\}, \{\as{x_2=3}, 190\},$ 
		\\ 
		\hspace*{1.0cm}$ 
		\{\as{x_3=1}, 40\}, \{\as{x_3=2}, 280\}, \{\as{x_3=3}, 230\},$ \\ 
		\hspace*{0.9cm} $\{\neg(x_1=x_2=x_3),\infty\}\}$
	\end{itemize}
	
	The notation $\as{x=a}$ is a predicate $p$ stating that variable $x$ is assigned to value $a$. Each constraint in
	$\cons$ is described with the notation $\{p,v_{i,p}\}$,
	 and  states that if predicate $p$ holds then a cost $v_{i,p}$ is paid by the Agent $A_i$ enforcing the constraint.

	One could attempt to model the privacy requirements by 
	aggregating the solution quality, related to  obtain the reward for a solution (called \solutionCost)   and for keeping the privacy  (called \privacyCost) into a unique cost. However, this is not 
	possible. Indeed, in a \DCOP, agents explore the search space to find a better solution, and only pay the corresponding solution cost 
	when the search is over and  the solution is accepted. 
	This means that the solution cost decreases with time. 
	However, privacy costs are cumulative and are paid during the 
	search process itself (at each time, a solution is proposed), 
	no matter what solution is accepted at the end of the computation. 
	This means that the total privacy loss increases with time. 
	Aggregating the solution costs and privacy costs or using a 
	multi-criteria \DCOP would not consider the privacy cost of the 
	solutions that are proposed but not kept as 
	final. Also, a given solution may imply different privacy losses 
	depending on the algorithm used to reach it. 	
	
	\newcommand{\comparesonprivacies}{
		\begin{table}[htbp]
			\caption{Synthesis of existing approaches to privacy} 
			\label{table:synthesisTable}			
			\centering
			\begin{tabular}
				{
					p{0.14\textwidth}p{0.00\textwidth}
					p{0.14\textwidth}p{0.14\textwidth}
					p{0.14\textwidth}p{0.14\textwidth}
				} \\ \toprule		
				&& \multicolumn{4}{l}{Privacy Type} \\  \cmidrule{3-6} 
				Framework && Value & Assignment & Constraint & Algorithmic \\ 
				\cmidrule{1-1} \cmidrule{3-6} 
				DisPrivCSP   && \cmark & \xmark & \cmark & \xmark \\
				VPS          && \xmark & \cmark & \xmark & \xmark \\
				PKC          && \xmark & \cmark & \xmark & \xmark \\
				SSDPOP       && \xmark & \cmark & \xmark & \xmark \\
				P-SyncBB     && \cmark & \xmark & \cmark & \xmark \\
				\bottomrule
			\end{tabular}
		\end{table}	
		
		Table~\ref{table:synthesisTable} sums up how existing approaches deal
		with different kinds of privacy loss in \DisCP. We see that none of them 
	}
	
	None of the previous techniques 
	consider all aspects of privacy, and while they preserve some 
	privacy, related algorithms require resources or properties from 
	the \DisCSP or \DCOP that may not be possible due to requirements dictated 
	by the real world problem. We therefore propose to deal with privacy 
	directly from the agents' decision making perspective, using a utility-based approach.

	\section{Utilitarian Distributed Constrained Problem}\label{Contribution}
	
	This section defines our utility-based framework (\ref{defi}),
	and then describes how it models previously presented problems,
	with different types of privacy requirements (\ref{description}).
	
	\subsection{Definitions} \label{defi}
	
	While some previously described frameworks do
	model the details of our example regarding privacy, it has until now been an open
	question as to how they can be dynamically used by algorithms in the
	solution search process.
	It can be noticed
	that the rewards and costs in our problem are similar to the
	utilities commonly used by planning 
	algorithms~\cite{koenig2002}. Thus, we propose to define a framework 
	that specifies the elements of the 
	corresponding family of planning problems.
	To do so, we ground the theory of our interpretation of 
	privacy in the well-principled theory of 
	utility-based agents~\cite{savauxPrivacite}.
	
	\begin{definition}
		A utility-based agent is characterized by its ability to associate 
		a value to each 
		state of the problem, representing its contentment to be in this state.
	\end{definition}
	
	For example, the {\it state} of Agent $A_i$ can include the subset of $\dom_i$ that it has
	revealed and the reward associated to the solution under consideration. 
	The problem is to define a policy for each agent such that their utility is maximized.
	A policy is a function that associates each state with an action that should be performed in it~\cite{russell1995}. We evaluate the utility of a state as follows:
	
	\begin{equation}
	Utility = \sum_i {reward}_i - \sum_j {cost}_j \text{ where } reward_i, cost_j \in \mathbb{R^+} \label{eq_utility}
	\end{equation} 
	
	Possible actions are communications dictated by 
	the used solver, and each possible solution proposal, acceptation or rejection, 
	is associated with its corresponding reward (for increasing the probability of  problem solving) and cost (for revealing 
	information).  
	
	Thus, unlike for \DisCPs, the solution of a \UDisCP  does not necessarily include an agreement,
	as pursuing the search may imply a decrease of utility.
	As privacy is lost during search, a given set of assignments can have different 
	utilities, depending on the information exchanged before.
	We define  Utilitarian Distributed Constrained Problem  (\UDisCP): 
	
	\begin{definition}
		A 
		Utilitarian Distributed Constrained Problem (\UDisCP) is formally defined  as a tuple $\langle \agents,\vars,\dom,\cons,\Util,\Rew \rangle$, \ie a \DisCP with:
		\begin{itemize}
			\item
			$\Util=\langle \Util_d,\Util_a,\Util_g,\Util_c \rangle$ is a quadruplet where: 
			\begin{itemize}
				\item $\Util_d$ is a matrix of domain privacy costs,
				where		$u_{d(i,j)}$ is the cost for $A_i$ to reveal whether
				$j\in \dom_{i}$. 
				\item $\Util_a$
				is a matrix of assignment privacy costs,
				where $u_{a(i,j)}$ is the cost for $A_i$ to reveal  
				a local solution including the assignment $x_i=j$. 
				\item $\Util_g$
				is a matrix of assignment privacy costs,
				where
				$u_{g(i,j)}$ is the cost for $A_i$ to reveal  
				the existence of a global solution including the assignment $x_i=j$. 
				\item $\Util_c$ is a vector of constraint privacy costs,
				where $u_{c(i,j)}$ is the cost  for $A_i$ to reveal 
				the weight of constraint $C_j$.
			\end{itemize}
			\item 
			$\Rew=\{ r_1,\dots,r_n\}$ is a vector of rewards,
			where $r_i$ is the reward that Agent $A_i$ receives if an agreement is found.
		\end{itemize}
	\end{definition}

	Note that is the case where a domain can be represented on a computer as a type of a parameter to a predefined software function modelling constraints (\eg an integer), then $\Util_d$ can be embedded in $\Util_c$, as a unary constraint on the domain. 
	We note also the notations \UDisCSP for Utilitarian \DisCSP and \UDCOP for Utilitarian \DCOP.

	\subsection{Description on Problems with Privacy} \label{description}
	
	In this section, we present how \UDisCP deals with the question of privacy 
	on the two previous examples, and we solve them using our utilitarian extensions 
	of existing solvers. For clarity, unchanged parts of pseudo-codes are shown in gray colour (while our extensions are in black colour). These parts may include variables or procedures that are used in other parts of the solvers but that are not used (and therefore not detailed) in this work. 
	
	\subsubsection{Privacy of Domains}
	
We propose to revise our previous example and present the extensions for two standard solvers (\SyncBT and \ABT).	
	
	\begin{example}\label{ex:udiscsp}
		Recall	$\Util_d$ represents the cost for each agent to reveal  possible values.
		$\Rew$ is the reward that each agent gets when a solution   is found, motivating them to initiate the solving.
		
		\DisCSP introduced in  Example~\ref{ex:discsp} is  extended to \UDisCSP 
		by specifying the additional parameters $\Util_d$ and $\Rew$ defined as follows.
		For example, $u_{d(1,3)}$ refers to the cost for the first agent to reveal the third value in its domain, namely $4$.
		
		\begin{equation*}
		\Util_d=
		\begin{blockarray}{*{3}{c} l}
		\begin{block}{*{3}{>{$\footnotesize}c<{$}} l}
		\end{block}
		\begin{block}{[*{3}{c}]>{$\footnotesize}l<{$}}
		1 & 2 & 4 &  $A_1$ \\
		1 & 2 & 4 &  $A_2$ \\
		1 & 2 & 4 &  $A_3$ \\
		\end{block}
		\end{blockarray}
		\hspace{2cm}	
		\Rew=\{4,5,5\}
		\end{equation*}
	\end{example}

	Now we discuss how the standard \ABT{} and \SyncBT algorithms are adjusted to \UDisCSPs.
	After each state change, each agent computes the estimated 
	utility of the state reached by each possible action, and selects randomly 
	one of the available
	actions leading to the state with the maximum expected utility.
		 
	\SyncBTU and \ABTU are obtained by 
	similar modifications of  
	\SyncBT and \ABT~\cite{yokoo1992,yokoo1998,zivan2003}, respectively.
	In our extended algorithms, agents compute the frequency of rejection of their 
	solution proposal to estimate the expected utilities. 
	This frequency can be re-evaluated at any moment based on data
	recorded during previous runs on problems of similar tightness
	(\ie having the same proportion of forbidden instantiations).
	Learning from previous experience has been extensively studied~\cite{Arbelaez2010,Arbelaez2012}. 	
	The learning can be off-line or on-line. 
	For off-line learning, agents calculate the number of 
	messages \okm{} and \nogoodm{} sent during previous executions, called $count$.
	They also count how many messages previously sent lead to the termination
	of the algorithm, in the variable $agreementCount$.
	The frequency with which a 
	solution leads to the termination of the algorithm, 
	is called $\risk{}$ (Equation~\ref{eq_agreementProb}). 
  	For on-line learning, one can update the variables $count$, $agreementCount$ and $\risk{}$
  dynamically whenever the corresponding events happen. When previous experiments are not available, 
  the value of $\risk{}$ is set to $\sfrac{1}{2}$  by default (this value is always between $0$ and $1$).

	\begin{equation} \label{eq_agreementProb}
	\risk{} = \frac{agreementCount}{count}
	\end{equation}

	When \okm{} messages are sent, the agent has the choice of which assignment
	to propose. When a \nogoodm{} message is scheduled to be sent, agents also
	have choices of how to express them. 
	Before each \okm{} or \nogoodm{} message, the
	agents check which available action leads to the highest expected utility. 
	If the highest expected utility is lower than the current one, the agent announces failure. 
	The result is used to decide between proposing assignments, a nogood, or declaring failure.
	
	\paragraph{Synchronous Backtracking with Utility (\SyncBTU)}
	
	\SyncBTU is obtained by restricting the set of  actions to the
	standard communicative acts of \SyncBT, namely \okm{} and \nogoodm{} messages.
	The procedures of a solver like \SyncBT define a policy, 
	since they only identify
	a set of actions (inferences and communications) 
	to be performed in each state. A state of an agent in \SyncBT is defined by \agentview{} 
	and a current assignment of the local variable.
	The local inferences in \SyncBTU are obtained 
	from the ones of \SyncBT by an extension
	exploiting the utility information available. 
	The criteria in this research was not to guarantee
	an optimal policy but to use utility with a minimal 
	change to the original behaviour of \SyncBT 
	reinterpreted as a policy.
	In \SyncBTU, the state is extended to also contain
	a history of revelations of one's values defining an accumulated privacy loss,
	and a probability to reach an agreement with each action. Since~\cite{yokoo1992} does not provide
	pseudo-code for \SyncBT, we modify 
	the pseudo-code
	presented in~\cite{zivan2003} for \SyncBTU: \assignCPA (before Line~$7$), 
	and \backtrack (before Line~$6$). As \SyncBTU is close from \ABTU, we do not detail 
	the pseudo-code for this algorithm in this report.
	
	\paragraph{Asynchronous Backtracking with Utility (\ABTU)}
	
	Similar modifications are applied 
	to \ABT to obtain \ABTU: communications  of \ABTU are 
	composed of \okm{}, \addlinkm{} and \nogoodm{}. The state and local inferences of \ABTU are similar to \SyncBTU, while also containing the set of nogoods.
	
	\begin{algorithm}
	\KwIn{$\dom_{self}$, \agentview, $agreementProb$, $r_{self}$}
	\SetKwFunction{algo}{algo}\SetKwFunction{proc}{proc}
	\SetKw{KwWhen}{when}
	\SetKw{KwDo}{do}
	\KwWhen $agentView$ and $currentValue$ are inconsistent \KwDo{} \\
	{	\setcounter{AlgoLine}{1}
		\color{gray}
		\uIf {no value in $\dom_{self}$ is consistent with \agentview{}}
		{\backtrack\;}
		\Else{
			select $d \in \dom_{self}$ where \agentview{} and $d$ are consistent\;
			$currentValue \leftarrow d$ \; \color{black}
		 \uIf {(\estimateCostDisCSP($\risk{}$, $\dom_{self}$, $1$) $\geqslant r_{self}$)}
			{\interruptSolving\;\label{ln:abtu10}}
			 \Else 
			{\color{gray} send (\textbf{ok?},($x_{self}$;$d$)) to outgoing links} 
		} 
	}
	\setcounter{AlgoLine}{0}
	\caption{\checkagentviewABTU}\label{algo:checkAgentViewABTU}	
\end{algorithm}

	To calculate the estimated utility of pursuing an agreement 
(revealing an alternative assignment), 
the agent considers all different possible scenarios of
the subsets of values that might have to be revealed in the future
based on possible rejections received, together with their probability
(Algorithm~\ref{algo:estimateCostDisCSP}).  This algorithm assumes as
parameters: (i) $\risk{}$ (Equation~\ref{eq_agreementProb}), (ii) an ordered set of possible values $\dom'_{self}$ for a scenario (by default, the order proposed in $\dom_{self}$), and (iii)  the probability to select a value from $\dom'_{self}$, initially $1$), called $\probabilityD$. Note that $\dom'_{self}[j]$ refers to value at index $j$ of $\dom'_{self}$.

	\begin{algorithm}
	\KwIn{$\risk{}$, $\dom'_{self}$, $\probabilityD$}
	\KwOut{\estcost{}}
	$valueId = j$ $|$ ($\dom_{self}[j] = \dom'_{self}[1]$)\;
	\uIf {($|\dom'_{self}| = 1$)}
	{\Return ($\sum_{j=1}^{j=valueId} u_{d(self,j)}$) $\times \probabilityD$\;}
	\Else
	{   	
		$v \leftarrow \dom'_{self}[1] $ \;
		\costTerminal $\leftarrow$ 
		\estimateCostDisCSP ($\risk{}$, \{$v$\}, $\risk{} \times \probabilityD$)\;
		$\CostTemp$ $\leftarrow$ 
		\estimateCostDisCSP ($\risk{}$, $\dom'_{self}\setminus\{v\}$, ($1-\risk{}$) $\times$ $\probabilityD$)\; 
		\estcost $\leftarrow$ $\costTerminal + \CostTemp $\; 
		\Return \estcost\; 
		\caption{{\estimateCostDisCSP}} \label{algo:estimateCostDisCSP}
	} 
\end{algorithm}

The algorithm then recursively computes the utility of the next possible
states, and whether the revelation of the current value $v$ leads to the
termination of the algorithm, values which are stored in variables $\costTerminal{}$ and $\CostTemp$. The algorithm returns the estimated cost (called \estcost) of privacy
loss for the future possible states currently.

	\begin{example}\label{ex:tree}
		
		Continuing with Example~\ref{ex:discsp} (whose a solving is illustrated by Figure~\ref{figure:ABT}), at the beginning of the solving,
		Agent \AgentOne{} has to decide for a first action to perform.
		We suppose the $\risk{}$ learned from previous solvings is $\numprint{0.5}$.
		To decide whether it should propose an available value or not, 
		it calculates the corresponding \estcost{} by calling 
		Algorithm~\ref{algo:estimateCostDisCSP} with parameters: 
		the learned $\risk{} = \numprint{0.5}$, 
		the set of possible solutions ($\dom'_{1}=\{1,2,3\}$) and 
		$\probabilityD=1$. 
		
		For each possible value, this algorithm recursively sums the cost for the 
		two scenarios corresponding to whether the action leads immediately to termination, or not.
		Given privacy costs, the availability of three possible subsets of $\dom'_{1}$ may be revealed in this problem: $\{1\}, \{1,2\}$, and $\{1,2,3\}$. Each set of size $S$ consists of  $S$ first elements of the list solution based on this initial order.
		 
		The \estcost{} returned is the sum of the costs for all 
		possible sets, weighted by the probability of their feasibility being revealed if an agreement is pursued. \\
		At the function call: 
		$\costTerminal=u_{d(1,1)} \times \numprint{0.5} = 1\times 0.5= 0.5$. \\
		At the next recursion: 
		$\costTerminal=(u_{d(1,1)}+u_{d(1,2)}) \times \numprint{0.25} = (1+2)\times0.25=0.75$. \\
		At the last recursion: $\costTerminal=(u_{d(1,1)}+u_{d(1,2)}+u_{d(1,3)}) \times \numprint{0.25}=(1+2+4)\times0.25=1.75$. 
		The algorithm returns the sum of these three values: $\estcost = $ 
		$0.5+0.75+1.75=3$. \\
		The expected utility  of pursuing a solution being positive 
		($reward - \estcost =4-3=1$), the first value is proposed.
		
	\end{example}	
	
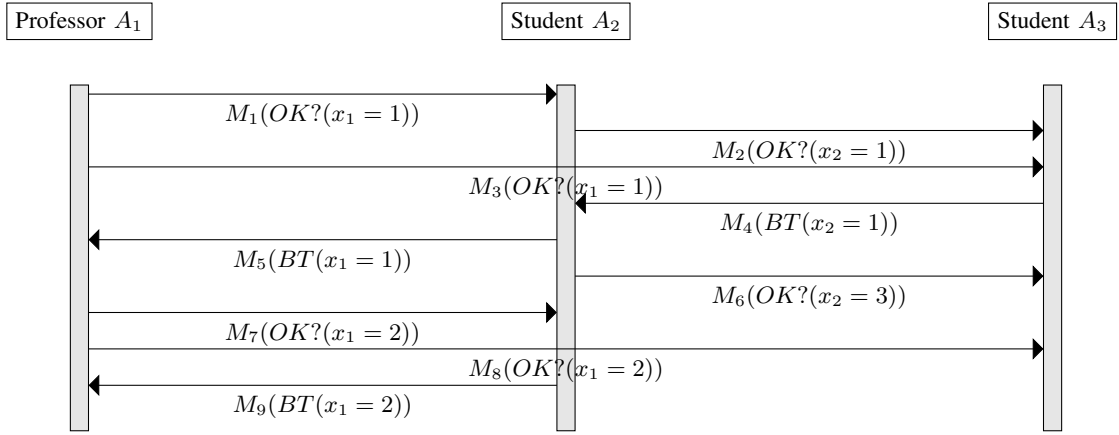
\begin{figure}
	\begin{center} 
		\begin{tikzpicture}[every node/.style={font=\small,}]
		\node [matrix, very thin,column sep=1.5cm, row sep=0.25cm] (matrix) at (0,0) {
			& \node(0,0) (\Aa) {}; 						&&&&			 		\node(0,0) (\Ab) {}; 					&&&& 					\node(0,0) (\Ac) {}; & \\ 		
			& \node(0,0) (\Aa 0) {}; 					&&&&		 			\node(0,0) (\Ab 0) {}; 					&&&&		 			\node(0,0) (\Ac 0) {}; & \\ 
			& \node(0,0) (\Aa 1) {}; 			&&\node(0,0) (\M 1) {};&& 		\node(0,0) (\Ab 1) {}; 					&&&& 					\node(0,0) (\Ac 1) {}; & \\ 
			& \node(0,0) (\Aa 2) {}; 					&&&&		 			\node(0,0) (\Ab 2) {}; 			&&\node(0,0) (\M 2) {};&&		\node(0,0) (\Ac 2) {}; & \\ 
			& \node(0,0) (\Aa 3) {}; 					&&&&		 			\node(0,0) (\M 3) {}; 					&&&&		 			\node(0,0) (\Ac 3) {}; & \\ 
			& \node(0,0) (\Aa 4) {}; 					&&&&		 			\node(0,0) (\Ab 4) {}; 			&&\node(0,0) (\M 4) {};&&		\node(0,0) (\Ac 4) {}; & \\ 
			& \node(0,0) (\Aa 5) {}; 			&&\node(0,0) (\M 5) {};&&		\node(0,0) (\Ab 5) {};					&&&&					\node(0,0) (\Ac 5) {}; & \\ 
			& \node(0,0) (\Aa 6) {}; 					&&&&					\node(0,0) (\Ab 6) {}; 			&&\node(0,0) (\M 6) {};&&		\node(0,0) (\Ac 6) {}; & \\ 
			& \node(0,0) (\Aa 7) {}; 			&&\node(0,0) (\M 7) {};&&		\node(0,0) (\Ab 7) {};					&&&&					\node(0,0) (\Ac 7) {}; & \\ 
			& \node(0,0) (\Aa 8) {};					&&&&					\node(0,0) (\M 8) {}; 					&&&&					\node(0,0) (\Ac 8) {}; & \\ 
			& \node(0,0) (\Aa 9) {};			&&\node(0,0) (\M 9) {};&&		\node(0,0) (\Ab 9) {}; 					&&&&					\node(0,0) (\Ac 9) {}; & \\ 
			& \node(0,0) (\Aa 10) {}; 					&&&&					\node(0,0) (\Ab 10) {}; 				&&&&		 			\node(0,0) (\Ac 10) {}; & \\ 	
		};
		
		\fill 
		(\Aa) node[draw,fill=white] {\Aaa}
		(\Ab) node[draw,fill=white] {\Aab}
		(\Ac) node[draw,fill=white] {\Aac};	
		
		\filldraw[fill=gray!20]
		(\Aa 1.north west) rectangle (\Aa 10.south east)
		(\Ab 1.north west) rectangle (\Ab 10.south east)
		(\Ac 1.north west) rectangle (\Ac 10.south east);	
		
		\draw [-latex, -triangle 90] (\Aa 1) -- (\Ab 1);
		\draw [-latex, -triangle 90] (\Ab 2) -- (\Ac 2);
		\draw [-latex, -triangle 90] (\Aa 3) -- (\Ac 3);
		\draw [-latex, -triangle 90] (\Ac 4) -- (\Ab 4);
		\draw [-latex, -triangle 90] (\Ab 5) -- (\Aa 5);
		\draw [-latex, -triangle 90] (\Ab 6) -- (\Ac 6);
		\draw [-latex, -triangle 90] (\Aa 7) -- (\Ab 7);
		\draw [-latex, -triangle 90] (\Aa 8) -- (\Ac 8);	
		\draw [-latex, -triangle 90] (\Ab 9) -- (\Aa 9);	
		
		\fill
		(\M 1) 
		node[below] {$M_{1} (OK?(x_{1}=1))$}
		(\M 2) 
		node[below] {$M_{2} (OK?(x_{2}=1))$}
		(\M 3) 
		node[below] {$M_{3} (OK?(x_{1}=1))$}
		(\M 4) 
		node[below] {$M_{4} (BT(x_{2}=1))$}
		(\M 5) 
		node[below] {$M_{5} (BT(x_{1}=1))$}
		(\M 6) 
		node[below] {$M_{6} (OK?(x_{2}=3))$}
		(\M 7) 
		node[below] {$M_{7} (OK?(x_{1}=2))$}
		(\M 8) 
		node[below] {$M_{8} (OK?(x_{1}=2))$}
		(\M 9) 
		node[below] {$M_{9} (BT(x_{1}=2))$};			
		\end{tikzpicture}
	\end{center}
	\caption{Interactions between agents during \ABT}\label{figure:ABT}
\end{figure}

	Next is an illustrative example of other \ABTU{} operations with the scenario $\{1,3,2\}$.

	\begin{example}	\label{ex:abtu}
		With the original \ABT{}, \AgentTwo{} proposes $x_2=1$ in 
		Message $M_2$ and $x_2=3$ in Message $M_6$.
		In this case, the privacy loss for \AgentTwo{} is $u_{d(2,1)}+u_{d(2,3)}=1+4=5$.		
		However, with \ABTU{}, we do not only use the actual utility of the
		next assignment to be revealed, but we estimate privacy loss using 
		Algorithms~\ref{algo:checkAgentViewABTU} and~\ref{algo:estimateCostDisCSP}. After \AgentTwo{} has already sent $x_2=1$ with $M_2$, it considers sending $x_2=3$ with
		$M_6$. This decision making process is depicted in Figure~\ref{figure:tree}. 
		If the next value, $2~pm$, is
		accepted, \AgentTwo{} will reach the final state while having
		revealed $x_2=1$ and $x_2=3$, for a total privacy cost of $u_{d(2,1)} + u_{d(2,3)} = 1+4= 5$. If it is not,
		the unavailability of the last value $x_2=2$ will have to be revealed
		to continue the search process, leading to the revelation of all its
		assignments for a total cost of $7$. Since both these scenarios have a probability
		of $50\%$ to occur,  \estcost{} equals
		$\sfrac{(5+7)}{2}=6$. The utility ($reward - \estcost{}$) being equal 
		to $5-6=-1$, \AgentTwo{}
		has no interest in revealing $x_2=3$ and interrupts the solving. Its
		final privacy loss is only $u_{d(2,1)}=2$. The utility of 
		the final state		
		reached by \AgentTwo{} being $-2$ with \ABTU{}, and $-4$ with
		\ABT{}, \ABTU{} preserves more privacy than \ABT{} in this problem. 
	\end{example}

	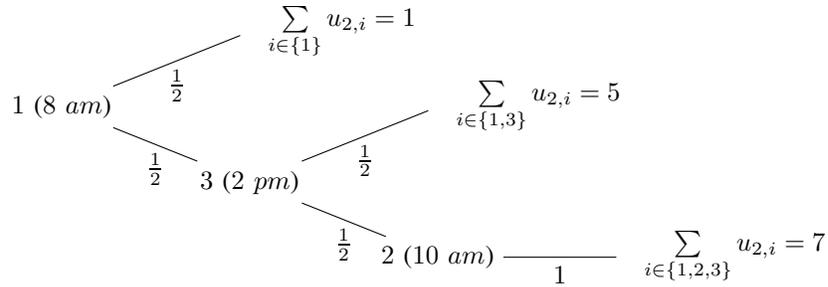
\begin{figure}
	\begin{center}		
		\scalebox{1.0}{
			\tikzstyle{level 1}=[level distance=2.5cm, sibling distance=2cm]
			\tikzstyle{level 2}=[level distance=2.5cm, sibling distance=2cm]
			\tikzstyle{bag} = []
			\tikzstyle{end} = []
			\begin{tikzpicture}[grow=right]
			\node[bag] {$1$ ($8~am$)}
			child {
				node[bag] {$3$ ($2~pm$)}        
				child {	node[bag] {$2$ ($10~am$)}    
					child{
						node[end, label=right:{$\sum\limits_{i\in \{1,2,3\} } u_{2,i}=7$}] {}       		
						edge from parent         
						node[above] {}
						node[below]  {$1$}
					}
					edge from parent         
					node[above] {}
					node[below]  {$\frac{1}{2}$}
				}
				child {
					node[end, label=right:{$\sum\limits_{i\in \{1,3\} } u_{2,i}=5$}] {}       		
					edge from parent         
					node[above] {}
					node[below]  {$\frac{1}{2}$}
				}					
				edge from parent         
				node[above] {}
				node[below]  {$\frac{1}{2}$}
			}
			child {
				node[end, label=right: {$\sum\limits_{i\in \{1\} } u_{2,i}=1$}] {}
				edge from parent
				node[above] {}
				node[below]  {$\frac{1}{2}$}
			};
			\end{tikzpicture}	
		}	
		\caption{Cost estimation for all scenarios by \AgentTwo{} during \ABTU}\label{figure:tree}	
	\end{center}
\end{figure}
	
	\subsubsection{Privacy of Assignments}	

We would like illustrate how our approach deals with  privacy of assignment in \UDCOP by extending the previous example and solving it with \ADOPTU.

	\begin{example}\label{ex:udcop}
		 \DCOP{} in Example~\ref{ex:dcop} is extended to a \UDCOP{}
		by specifying the additional parameters $\Util_a$ and $\Rew$.
		$\Util_a$ represents the cost for each agent to reveal each assigned value.
		$\Rew$ is a default reward (for example, $1000$) that each agent gets for finding a solution to the problem, motivating them to initiate the solving. 		
		The description is as follows:
		
		\begin{equation*}
		\Util_a=
		\begin{blockarray}{*{3}{c} l}
		\begin{block}{*{3}{>{$\footnotesize}c<{$}} l}
		\end{block}
		\begin{block}{[*{3}{c}]>{$\footnotesize}l<{$}}
		80 & 20 & 40 & $A_1$ \\
		130 & 30 & 10 &  $A_2$ \\
		80 & 30 & 10 &  $A_3$ \\
		\end{block}
		\end{blockarray}
		\hspace{2cm}	
		\Rew=\{1000,1000,1000\}
		\end{equation*}		
		
	\end{example}
	
	Now we discuss how the standard \DCOP{} algorithms are adjusted 
	to \UDCOPs{}. 
	After each state transition, each agent computes the estimated 
	utility of the state reached by each possible action, and selects 
	randomly one of the actions leading to the state with the 
	maximum expected utility.
	In our algorithms, agents estimate expected utilities using the risk of one of their assignments to 
	not be a part of the final solution. 

To estimate the cost for a \DCOP solution, \estimateCostDCOP is introduced (Algorithm~\ref{algo:estimateCostDCOP}). Its inputs are the utilities considered ($\Util_i$), the domain of possible values ($\dom_{self}$) and the already revealed informations 
($revealedInfos$). For each solution, the algorithm evaluates \estcost including both constraint costs (\ie \solutionCost{} and \privacyCost). Recall that the utility of the reached state is calculated by using Equation~\ref{eq_utility}, and equals the fixed reward (by hypothesis) reduced by \estcost. If the initial \DCOP is a maximization problem, it is 
first recast as a minimization one, 
so that constraint values correctly belong to costs.   

\begin{algorithm}[]
	\KwIn{$utilities$, $\dom_{self}$, $revealedInfos$} 
	\KwOut{\estcost}
	\solutionCost{} $\leftarrow 0$\;
	\privacyCost{} $\leftarrow 0$\;
	\ForEach{value $d \in \dom_{self}$}
	{\ForEach{constraint $c \in \cons_{self}$}
		{
			\If{($c$ contains predicate $p$ such as $p= (x_i = d)$) $\wedge$ ($d \in revealedInfos$)}	
			{
				\solutionCost{} $\leftarrow$ \solutionCost{} $+$ $v_{self,p}$ \;
				\privacyCost{} $\leftarrow$ \privacyCost{} $+ utilities_{(self,d)}$\;
			}				
		}
	}
	\estcost $\leftarrow$ \solutionCost{} $+$ \privacyCost{}\;
	return \estcost\;
	\caption{\estimateCostDCOP}\label{algo:estimateCostDCOP}
\end{algorithm}
	
	Before proposing a new value, agents estimate the utility that will 
	be reached in the next state. This value is the summation of 
	the costs of revealed \agentview{}s 
	(weighted by their probability to be the final solution) 
	in the said state, and of the corresponding privacy costs.
	If this utility is lower than the estimation of the current state,
	the agent proposes the next value, otherwise it keeps its current value.
	
	\paragraph{Asynchronous Distributed Optimization with Utility (\ADOPTU)}
	is a method obtained
	from \ADOPT{} by adding Lines~$7$ to $10$  in Algorithm~\ref{algo:adoptu} (procedure \checkAgentViewADOPTU).
	At Line~$7$, the possible next value is set to the value that has the minimal cost.
	The cost reached after the next value (Line~$8$) 
	and the cost of the current state (Line~$9$) are estimated.
	At Line~$10$, if the next cost is lower than the current cost, the maximal improvement and the new value are updated.

	\begin{algorithm}
	\KwIn{$utilities$, $\dom_{self}$, $revealedAssignments$}
	\SetKw{KwForE}{for each}
	\SetKw{KwDo}{do}
	\color{gray}
	\KwForE{} ($d \in \dom_{self}$) update $l[d]$ and recompute $h[d]$\;
	\KwForE{} $A_j$ with higher priority than $A_{self}$ 
	\KwDo{}\;
	\If {($h$ has non-null cost CA for all values of $\dom_{self}$)}
	{
		$vn \leftarrow$ min\_resolution($j$)\;
		\If {$vn \neq lastSent[j]$}
		{
			send \textbf{nogood}($vn$,$self$) to $A_j$\;
		}	
	}
	$newValue \leftarrow$  argmin$_d$($cost[h[d]]$)\;
	\color{black}
	$nextCost \leftarrow$ \estimateCostDCOP($utilities$, $\dom_{self}$, $revealedAssignments \cup newValue$)\;
	$currentCost \leftarrow$ \estimateCostDCOP($utilities$, $\dom_{self}$, $revealedAssignments$)\;
	\If {($nextCost < currentCost$)}
	{
			\color{gray}
		assign $newValue$ to $x_{self}$ \; 
		\If {($x_{self}$ was modified)} 
		{
			send \textbf{ok?} to each neighbour to inform them about the value change\;
		}
	}
	\caption{\checkAgentViewADOPTU}\label{algo:adoptu}	
\end{algorithm} 
	
	\begin{example}		\label{ex:adoptu}
		Continuing with Example~\ref{ex:dcop}, 
		at the beginning of the computation with 
		the  \ADOPTU{} solver, 
		the participants select a random value.
		The resulting set of assignments is $x_1=1$, $x_2=3$ and $x_3=2$
		for \AgentOne{}, \AgentTwo{}, and \AgentThree{} respectively. 
		The participants then inform their values to their linked descendants 
		(\AgentTwo{} and \AgentThree{} for \AgentOne{}, and \AgentThree{} for \AgentTwo).
		
		\AgentTwo{} will have  $\agentview{}=\{(x_1=1)\}$
		and chooses the value $1$, since this value minimizes its cost.
		\AgentThree{} sends a VIEW message with cost of $190$ to \AgentTwo{}.
		This cost is reported because given the value of $x_1$, it is a lower bound on global solution cost:
		$190$ is a lower bound on local cost at $x_2$, and $0$ is a lower bound on local costs at other variables.	
		Now concurrently with the previous execution of $A_2$, 
		$A_3$ receives the assignment $x_2=3$.
		With \ADOPTU{}, $A_3$ realizes that changing its value would increase
		its cost and decides to keep its value unchanged.
		However, with \ADOPT, $A_3$ decides to change its value to $3$, as 
		it reduces the cost.	
		Then, with \ADOPT and \ADOPTU, $A_3$ 
		receives the assignments $x_1=1$ and $x_2=1$, 
		and decides to change its value to $x_3=1$.
		At the final step, the previous \agentview{} is the
		optimal solution. With \ADOPTU{}, the reached costs are:
		${70+80=150, 120+30+100=250, 80+80=160}$ 
		for \AgentOne{}, \AgentTwo{}, and \AgentThree{} respectively. 
		With standard \ADOPT, the final utilities are:	
		${70+80=150,120+30+100=250,80+80+30=190}$
		for \AgentOne{}, \AgentTwo{}, and \AgentThree{} respectively. 
		Therefore, by using \ADOPTU{} instead of \ADOPT{}, \AgentThree{} avoids a futile revelation of $x_3=2$
		and reduces its utility by $30$.  
	\end{example}
	
	\subsubsection{Privacy of Constraints}
		
Similarly, we detail the model for our example and we present the extensions of two solvers for \UDCOPs (namely \DBOU and \DSAU).
	
	\begin{example}\label{ex:udcoppc}	
		\DCOP{} described in Example~\ref{ex:dcop} is extended to a \UDCOP{} by specifying the
		parameters $\cons$, $\Util_c$ and $\Rew$ as follows.
		$\cons$ is the updated set of constraints.
		$\Util_c$ represents the cost for each agent to reveal each assignment cost.
		$\Rew$ is a default reward that each agent gets for finding a solution to the problem, motivating them to initiate the solving. 	
		$\cons$'s subsets known to each agent are:\\
		\noindent
		$\cons_1=\{ 
		c_{1,1}=[\as{x_1=1}, 70], 
		c_{1,2}=[\as{x_1=2}, 230],$\\ \hspace*{1.05cm} $
		c_{1,3}=[\as{x_1=3}, 270],
		c_{1,4}=[\neg(x_1=x_2=x_3),\infty]$\}
		\\
		$\cons_2=\{ 
		c_{2,1}=[\as{x_2=1}, 120],
		c_{2,2}=[\as{x_2=2}, 400],$\\ \hspace*{1.05cm} $
		c_{2,3}=[\as{x_2=3}, 190],
		c_{2,4}=[\neg(x_1=x_2=x_3),\infty]$\} 
		\\ 
		$\cons_3=\{ 
		c_{3,1}=[\as{x_3=1}, 40], 
		c_{3,2}=[\as{x_3=2}, 280],$\\ \hspace*{1.05cm} $
		c_{3,3}=[\as{x_3=3}, 230],
		c_{3,4}=[\neg(x_1=x_2=x_3),\infty]\}
		$ 
		\begin{equation*}
		\Util_c=
		\begin{blockarray}{*{3}{c} l}
		\begin{block}{*{3}{>{$\footnotesize}c<{$}} l}
		\end{block}
		\begin{block}{[*{3}{c}]>{$\footnotesize}l<{$}}
		80 & 20 & 40 & $A_1$ \\
		100 & 30 & 10 &  $A_2$ \\
		80 & 30 & 10 &  $A_3$ \\
		\end{block}
		\end{blockarray}
		\hspace{2cm}	
		\Rew=\{1000,1000,1000\}
		\end{equation*}
		
	\end{example}
	
	To adapt existing algorithms for privacy of constraints,
	revealed domains and possible revealed domains are changed to the
	revealed constraints and possible revealed constraints, respectively.

	\paragraph{Distributed Breakout with Utility (\DBOU)}
	is a solver obtained
	from \DBO{} by adding Lines~$5$ to $11$ in Algorithm~\ref{algo:dbou} (procedure \sendImproveDBOU).
	At each iteration, \DBOU does not only uses solution cost to guide search and 
	computes the value giving maximal improvement as in standard \DBO (Lines~$2$ to $4$), 
	but also considers constraint privacy costs (Lines~$5$ to $7$).
	The new chosen value is the one minimizing total cost (Lines~$8$ to $11$).
	As privacy loss is cumulative, agents update the set of revealed constraints 
	to also consider previously revealed constraints during their estimation of 
	the reached cost for the different considered values.

	\begin{algorithm}[t]
	\KwIn{$utilities$, $\dom_{self}$, $revealedConstraints$}
	\SetKwFunction{algo}{algo}\SetKwFunction{proc}{proc}
	\SetKw{KwWhen}{when}
	\SetKw{KwDo}{do}
	\color{gray}
	$ eval \leftarrow$ evaluation value of $x_{self}$\;	
	
	$myImprove \leftarrow 0$ \; 
	$newValue \leftarrow x_{self}$\;
	
	$nextValue \leftarrow$ value giving maximal improvement\;
	\color{black}
	$nextRevealedConstraints \leftarrow revealedConstraints$ $\cup$ constraints with $nextValue$ \;
	$nextCost \leftarrow$ \estimateCostDCOP($utilities$, $\dom_{self}$, $nextRevealedConstraints$)\;
	$currentCost \leftarrow$ \estimateCostDCOP($utilities$, $\dom_{self}$, $revealedConstraints$)\;
	\If {($nextCost < currentCost$)}
	{
		$myImprove \leftarrow$ possible max improvement\;
		$newValue \leftarrow$ value giving maximal improvement\;
		 $x_{self} \leftarrow newValue$ \;
	}
   \color{gray}
	\uIf {($eval = 0$)}
	{$consistent \leftarrow true$}
	\Else{
		$consistent \leftarrow false$ \;
		}
	$terminationCounter \leftarrow 0$ \;
	\uIf {($myImprove > 0$)}
	{$canMove \leftarrow true$ \;
		$quasiLocalMin \leftarrow false$ \;}
	\Else	
	{$canMove \leftarrow false$ \;
		$quasiLocalMin \leftarrow true$}
	send (\textbf{improve}, $x_{self}$, $myImprove$, $eval$, $terminationCounter$) to neighbours\;
		\color{black}
	\setcounter{AlgoLine}{0}
	\caption{\sendImproveDBOU}\label{algo:dbou}
\end{algorithm} 
	
	\paragraph{Distributed Stochastic Algorithm with Utility (\DSAU)}
	is obtained from standard \DSA
	by adding  Lines~$10$ to $13$ in
	Algorithm~\ref{algo:dsau} (procedure \procDSAU). Each agent computes a solution and sends it to its neighbours (Lines~$1$ to $5$). At each iteration, after collecting new values from neighbours (Line~$6$), each agent compares the new 	
	{\agentview} with the previous one (Line $7$). If a difference is detected, the agent $A_{self}$ computes a new solution considering both solution (Lines~$8$ and $9$) and privacy costs (Lines ~$10$ to $13$) similarly with \ADOPTU and \DBOU.  

	\begin{algorithm}
	\KwIn{$utilities$, $\dom_{self}$, $revealedConstraints$}
	\SetKwFunction{algo}{algo}\SetKwFunction{proc}{proc}
	\SetKw{KwWhen}{when}
	\SetKw{KwDo}{do}
	\color{gray}
	$nextValue \leftarrow$ randomly chosen a value\;
	\While{(no termination condition is met)} 
	{
		\If{($nextValue \neq x_{self}$)}
		{   $x_{self} \leftarrow nextValue$\;
			send $x_{self}$ to neighbours\;}	
		$nextView \leftarrow$  collect $x_i$ for each neighbour $A_i$\; 
		\If{($nextView \neq$ \agentview) } 
		{
			\agentview $\leftarrow nextView$\;	
		$tempValue \leftarrow$ randomly chosen value\;
		\color{black}
		 $revealedConstraints \leftarrow revealedConstraints \cup \cons_{self}$ with $x_{self} = tempValue$\;
		$nextCost \leftarrow$ \estimateCostDCOP  ($utilities$, $\dom_{self}$, $nextRevealedConstraints$)\;
		$currentCost \leftarrow$ \estimateCostDCOP($utilities$, $\dom_{self}$, $revealedConstraints$)\;
		\If {($nextCost < currentCost$)}
			{\color{gray}
			  $x_{self} \leftarrow tempValue$\;
			}
		}
	}
	\color{black}
	\setcounter{algoline}{0}
	\caption{DSAU}\label{algo:dsau}
\end{algorithm}
	
	\begin{example}		\label{ex:dsau}
		Continuing with Example~\ref{ex:dcop}, at the beginning of the
		computation with the \DSAU{} solver, the participants select a random
		value. The resulting \agentview{} of each agent is :\\ $\{(x_1=1),(x_2=1),(x_3=3)\}$.  The utilities of the reached state are: \\
		${v_{1,1}+u_{c(1,1)}=70+80=150}$, \\
		${v_{2,1}+u_{c(2,1)}=120+100=220}$, and \\
		${v_{3,3}+u_{c(3,3)}=230+10=240}$ for \Students{} respectively. \\
		The	participants then inform each other of their value. They 
		consider changing them to a new randomly selected one. The
		considered \agentview{} is $\{(x_1=2),(x_2=3),(x_3=1)\}$.
		If the participants
		change their value, the utilities of the reached states would be: \\
		${(v_{1,1}+v_{1,2})/2+u_{c(1,1)}+u_{c(1,2)}=250}$,\\
		${(v_{2,1}+v_{2,3})/2+u_{c(2,1)}+u_{c(2,3)}=265}$, \\
		${(v_{3,3}+v_{3,1})/2+u_{c(3,3)}+u_{c(3,1)}=225}$, for \Students{} respectively. \\ \AgentOne{} and 
		\AgentThree{} do not propose the new value as it would increase their
		utility. 
		However, \AgentThree{} chooses to change its value from $2$
		to $1$ which lowers its utility from $240$ to $225$. In the next step, 
		\agentview{} is $\{(x_1=1),(x_2=1),(x_3=1)\}$.
		Participants then do not change
		their value any-more, as all other options would not decrease the
		utility.  At the final step, the previous \agentview{} is therefore the
		optimal solution. With \DSAU{}, the reached utilities are : \\
		${70+80=150,120+100=220,40+10+80=130}$  for \Students{} respectively. 
		With standard \DSA, the final utilities are: \\ 
		${(v_{1,1}+u_{c(1,1)}+u_{c(1,2)}+u_{c(1,3)})=230}$, \\
		${(v_{2,1}+u_{c(2,1)}+u_{c(2,2)}+u_{c(2,3)})=260}$, \\
		${(v_{3,1}+u_{c(3,2)}+u_{c(3,1)}+u_{c(3,3)})=160}$, for \Students{} respectively. \\
		Therefore, using \DSAU{} instead of \DSA{} reduces the utility 
		by ${80, 40, 30}$. 
	\end{example}

	In this work, studied problems include only one type of privac
	at a time, to illustrate 
	proposed models and algorithms with simple examples.
	However, problems integrating several types of privacy can also be modelled 
	with \UDisCP. Such problems, where agents would have optimize multiple objectives,
	will be investigated in future works.
	
	\section{Theoretical Discussion}\label{TheoreticalDiscussion}
	
	This section deals with three theoretical studies, \ie  comparisons with \DCOPs{}  (\ref{Compardiscsp}),  \MODCOPs{} (\ref{ComparMODCOP}) and \POMDP{} (\ref{ComparPOMDP}).
	
	\subsection{Comparison with \DCOPs} \label{Compardiscsp}
	The introduced \UDCOP{} framework can assume  that 
	inter-agent constraints are public (without 
	significant loss of generality). This is due 
	to the fact that any problem with private 
	inter-agent constraints, is equivalent with its dual 
	representation where each constraint 
	becomes a variable~\cite{bacchus1998}. 
	
	\begin{theorem}
		\UDCOP{} planning and execution is at least as general as \DCOPs{} solving.
	\end{theorem}
	
	\begin{proof}
		A \DCOPs{} can be modelled as a \UDCOPs{} with all privacy costs equal $0$. The obtained 
		\UDCOPs{} would always reach an agreement, if possible. Therefore the goal of a \UDCOPs{} would also match with the goal of the modelled \DCOPs{}. 
		This implies a tougher class of complexity for \UDCOPs{}.
	\end{proof}
	
	The space complexity required by \ABTU{} and \SyncBTU{} in each agent is identical with 
	the one of \ABT{} and \SyncBT{}, since the only additional structures are 
	the privacy costs associated with its values, constituting a constant factor increases for domain storage 
	Similarly, additional structures with constant values are added from \DCOP to \UDCOP. 
	
	\subsection{Comparison with \MODCOPs} \label{ComparMODCOP}
	
	\paragraph{Multi-Objective Distributed Constraint Optimization Problem}
	A multi-objective optimization problem (\MOOP{})~\cite{maheswaran2004,delle2011} 
	is defined as the problem of simultaneously maximizing $k$ objective 
	functions that have no common measure, 
	defined over a set 
	of variables, each one taking its value in a given domain.
	Thus, a solution to  \MOOP{} is a set of assignments maximizing 
	the combination of the objective functions.
	Here, each objective function 
	can be defined over a subset
	of variables of the problem. 
	However, to simplify our discourse, we assume that each function is defined over the
	same set of variables. A Multi-Objective \DCOP~\cite{clement2013}, called
	\MODCOP{}, 
	is an extension of the standard mono-objective \DCOPs{}. 
	
	Note that a \MODCOP{} is a \DCOP{} where the weight of each constraint
	tuple is a vector of values $\vas{w_i}$, each value $w_i$ representing
	a different metric.  Two weights $\vas{w^1_i}$ and $\vas{w^2_i}$ for the
	same partial solution, inferred from disjoint sets of weighted
	constraints, are combined into a new vector $\vas{w^3_i}$ where each
	value is obtained by summing the values in the corresponding position
	two input vectors, namely $w^3_i=w^1_i+w^2_i$.  The quality of
	a solution of \MODCOP{} is a vector integrating the cost of all
	weighted constraints. The vectors can be compared using various
	criteria, such as leximin, maximin,
	social welfare or Theil index~\cite{netzer2011,matsui2015}.
	
	To clarify why Multi-Objective \DCOPs{} (\MODCOPs{}) cannot integrate our
	concept of privacy as one of the criteria they aggregate, we give an
	example of what would be achieved with \MODCOPs{}, as contrasted with the
	results using \DCOPs{}. 
	We show a comparative trace based on one of
	the potential techniques in \MODCOPs{}, providing a hint on why
	\MODCOPs cannot aggregate privacy lost during execution in the same
	way as \UDCOP. In this example, the privacy value of each assignment
	and its constraint cost are two elements of an ordered pair defining
	the weight of  \MODCOP. For illustration, in this example pairs of
	weights are compared lexicographically with the privacy having
	priority.

	\begin{example}			\label{ex:modcopdsa}
		We assume to model  Example~\ref{ex:dcop} with a
		\MODCOP{}. As also illustrated in the trace (Table~\ref{tab:modcop}) with lexicographical comparison (\UDCOP{} \DSAU{} vs. \MODCOP{} \DSA{}), privacy first.  Candidate values are marked with $^*$ if they are better than old values, and will be adopted. For these two approaches, we evaluate both \solutionCost{} and \privacyCost.
		
		At the beginning of  \DSA process, the
		participants select a random value. Resulting \agentview{} is $\{(x_1=1),(x_2=1),(x_3=3)\}$. For instance, \AgentOne evaluates this state with \privacyCost{} of $80$ and \solutionCost{} of $70$.   Recall that privacy criterion is prevailing; hence a notation $[80,70]$. \DSAU evaluates this same state as a summation of these costs, with a value of $150$. The participants then inform each others of their value.
		 They  consider a  change of their value to a new randomly selected one.   \agentview{} is $\{(x_1=2),(x_2=3),(x_3=1)\}$ for these two solvers. 
		
		With \UDCOP{} \DSAU{} (state$_1$), 	\AgentOne and \AgentTwo{}
		do not propose the new value as it
		would increase their cost, and \AgentThree{} chooses to change
		its variable's value from $3$ to $1$.
		However, with \MODCOP{} \DSA{}, \AgentTwo{} changes its value to $3$ (as it believed \privacyCost{} will drop from $100$ to $10$),
		which is not the case with \DSAU{}, which implies privacy
		loss. The \agentview{} is now $\{(x_1=2),(x_2=3),(x_3=3)\}$ (the real privacy loss is the summation of all revealed costs, \ie $100+10=110$, and not only $10$). 
		
		In short, with the \MODCOP{} model, \AgentTwo{} reveals more values and 
		loses more privacy (with a difference of $110-100=10$ for privacy costs) than with \UDCOPs{}.

	\end{example}

	\begin{table}
	\caption{Comparative trace of two rounds (\UDCOP{} \DSAU{} vs. \MODCOP{} \DSA{})}\label{tab:modcop}			
	\centering
	\begin{tabular}{p{0.2\textwidth}
			p{0.1\textwidth}p{0.1\textwidth}p{0.1\textwidth}
			p{0.00\textwidth}
			p{0.1\textwidth}p{0.1\textwidth}p{0.1\textwidth}
		} \\ \toprule		
		Framework & \multicolumn{3}{c}{\UDCOP{} \DSAU} && \multicolumn{3}{c}{\MODCOP{} \DSA} \\ \midrule
		Agent & $A_1$ & $A_2$ & $A_3$ && $A_1$ & $A_2$ & $A_3$ \\ \cmidrule{2-4} 
		\cmidrule{6-8} 
		\multicolumn{8}{c}{current state} \\ \midrule 
		state$_0$  & $1$ & $1$ & $3$ && $1$ & $1$ & $3$ \\
		\solutionCost  & $70$ & $120$ & $230$ && $70$ & $120$ & $230$ \\
		\privacyCost  & $80$ & $100$ & $10$ && $80$ & $100$ & $10$ \\ 
		evaluation  & $150$ & $220$ & $240$ && $[80,70]$ & $[100,120]$ & $[10,230]$ \\ 
		\multicolumn{8}{c}{believed next state} \\ \midrule 
		considered state & $2$ & $3$ & $1$ && $2$ & $3$ & $1$ \\
		\solutionCost  & $150$ & $155$ & $135$ && $230$ & $190$ & $40$ \\
		\privacyCost  & $100$ & $110$ & $90$ && $20$ & $10$ & $80$ \\ 
		evaluation  & $250$ & $265$ & $225$* && $[20,230]$* & $[10,190]$* & $[80,40]$ \\ 
		\multicolumn{8}{c}{achieved next state} \\ \midrule 
		state$_1$  & $1$ & $1$ & $1$ && $2$ & $3$ & $3$ \\
		\solutionCost  & $70$ & $120$ & $40$ && $230$ & $190$ & $230$ \\
		\privacyCost & $80$ & $100$ & $90$ && $100$ & $110$ & $10$ \\ 
		evaluation  & $150$ & $220$ & $130$ && $[100,230]$ & $[110,190]$ & $[10,230]$ \\ 
		\bottomrule
	\end{tabular}
\end{table}	
	
	The assumption that each agent owns a single variable is
	also not restrictive.  Multiple variables in an agent can be
	aggregated into a single variable by Cartesian product. 
	Nevertheless some algorithms can exploit these underlying structures
	for efficiency, and this has been the subject of extensive 
	research~\cite{fioretto2015,mandiau2014}.
	
	We now discuss how \UDisCP{} can be interpreted as a planning problem.
	
	\subsection{Comparison with \POMDPs} \label{ComparPOMDP}
	
	The problem that each agent in \UDisCP{} has to solve
	have similarities to a Partially Observable Markov
	Decision Problem (\POMDP{}). Given ways to approximate
	observation and transition conditional probabilities, these problems could
	be reduced to \POMDPs{}~\cite{monahan1982,kaelbling1998,Dibangoye2009}.
	
	A \POMDP agent regularly reasons in terms of belief (probability 
	distribution over the states), and tries to build a policy, namely a 
	recommendation of each action to be executed as function of current belief. 
	For \UDisCP{}, the corresponding \POMDP  is defined by the tuple 
	 $\langle S,A,T,R,\Omega,O,\gamma \rangle$ with components~\cite{savaux2016discsps}: 
	\begin{itemize}
		\item $\cal S$: Set of states of the agent, defined by possible 
		contents of its \agentview{}, of the nogoods stored by the agent, the knowledge the agent gathers about the secret elements of the \UDisCP{} unknown to it, and the 
		information  already revealed.
		\item $\cal A$: Set of actions available to an agent, consisting in local reasoning and communication actions		
	     that are a function of the
		selected protocols (\ie communication language).  For
		example, in \ABT{} these communications actions can have as payloads
		assignment announcements (\okm{} messages) and nogoods in (\nogoodm{} messages).
		\item $T$: Set of transition probabilities between states given actions
		for \UDisCP{}. It is estimated
		in our approach by training 
		$\risk{}$.  
		\item $\cal R$: Set of rewards of \POMDP{} is the same as $\Rew$ for the corresponding \UDisCP{}.
		\item $\Omega$: Set of possible observations is given by 
		$\Util$, the possible	
		 incoming payloads
		of the communication actions available in \UDisCP{},		
		as well as possible results of local reasoning steps.
		\item $O$: Set of conditional observation probabilities.  In reported experiments, it is assumed that the 
		message payloads truthfully reveal the corresponding elements of 
		the states of the other agents, while the probability of the 
		remaining elements have to be inferred by the agent.
		\item $\gamma$: Discount factor set to $1$, since we have not 
		taken into consideration the impact of time on utilities in this work.					
	\end{itemize}
	
	In the next section, we will present experiments that use our \UDisCP{} models and algorithms to preserve privacy.

	\section{Experimental Results}\label{Experiments}
		
	We describe the experimental protocol for our study (\ref{protocol}), and then we give the obtained results for DMS context (\ref{result}).
	
	\subsection{Experimental protocol and DMS context} \label{protocol}
	
	We evaluate our framework and algorithms on randomly generated
	instances of {\em distributed meeting scheduling problems (DMS)}~\cite{maheswaran2004,bessiere2007,gershman2008}. Existing studies has already addressed the question of
	privacy in DMS problems by considering the
	information on whether an agent can attend a meeting to be
	private~\cite{wallace2005,brito2008}.

	The algorithm we use to  generate the DMS problems is defined according 
	to the following procedure: 
	(i) Variables are created and associated with the agent controlling them. 
	(ii) Domains (possible values) are initialized for all variables. 
	(iii) The global constraint {\em ``all equals''} is added.
	(iv) The unary constraints (individual unavailabilities) are added.
	(v) Privacy costs uniformly distributed between $0$ and $9$ are generated for each possible variable value, variable assignment and constraint.
	
	The experiments are carried out on a computer under Windows $7$, using a $1$ core $\SI{2,16}{\giga\hertz}$ 
	CPU and $\SI{4}{\gibi\byte}$ of RAM. 
	Implementation is done in Java (jre $1.8$) using the JADE platform (version $4.3.3$) 
	to build the multi-agent system~\cite{bellifemine2005}.
	
	The problems are parametrized as follows: $10$, $20$, or $40$ agents, $10$ possible values, and the cost of a revelation is a random number between $0$ and $9$.
	These parameters are used to guarantee that the problems are not
	over-constrained or under-constrained, as the probability to find a solution increases 
	with the increase of domain size, and with the decrease of the number of agents and 
	unary constraint tightness. Density is defined as the proportion of 
	binary constraints
	in $\cons$ among $\vars_i \times \vars_j$ ($i \neq j$). Because of the constraint requiring all 
	couples of variables to be equals, density in DMS is always $1$. Tightness is defined as proportion of forbidden tuples among all constraints 
	(binary but also unary ones) of $\cons$. 
	
	\begin{proof}
		For DMS, we assume  $t$ for constraint tightness, $m$ for the number of agents $\agents$ and $d$ for the size of each domain set  $\dom_i$.
		Indeed, the probability that a given value is authorized is:
		
		\begin{equation}\notag 1-t \end{equation}
		
		The probability that a given value is authorized by all agents is:
		\begin{equation}\notag (1-t)^m \end{equation}
				
		The probability that a given value is forbidden by at least one agent is: 
		\begin{equation}\notag 1-(1-t)^m \end{equation}

		The problem has a probability $s$ to have at least one solution if and only if:
	
		\begin{equation}\notag 1-s = (1-(1-t)^m)^d \end{equation} 
		\begin{equation}\notag 1-\sqrt[d]{(1-s)}=(1-t)^m \end{equation} 
		\begin{equation}\notag 1-t=\sqrt[m]{(1-\sqrt[d]{(1-s)})} \end{equation}
		
	\end{proof}
	
	Figures~\ref{figure:parametersDMS}	
	corroborates this for a DMS modelled by 
	 \DisCSP with a probability of $50\%$ 
	to have at least one solution, as depicted according 
	to the previous formula and to experimental results. These figures are described by the number of agents (coordinate $x$), domain size (coordinate $y$) and tightness (coordinate $z$).

	\begin{figure}[h]
	\begin{center}
		\begin{tikzpicture}[scale=0.7]
		\begin{axis}[
		xlabel={$log_{2}(nbAgents)$},
		ylabel={$log_{2}(domainSize)$},
		zlabel={$tightness$}
		]
		\addplot3[surf,mesh/rows=7] coordinates {
			(1,1,46) (2,1,26) (3,1,14) (4,1,7) (5,1,4) (6,1,2) (7,1,1)  
			(1,2,60) (2,2,37) (3,2,21) (4,2,11) (5,2,6) (6,2,3) (7,2,1)
			(1,3,71) (2,3,47) (3,3,27) (4,3,14) (5,3,7) (6,3,4) (7,3,2)
			(1,4,79) (2,4,55) (3,4,33) (4,4,18) (5,4,9) (6,4,5) (7,4,2)
			(1,5,85) (2,5,62) (3,5,38) (4,5,21) (5,5,11) (6,5,6) (7,5,3)
			(1,6,90) (2,6,68) (3,6,43) (4,6,25) (5,6,13) (6,6,7) (7,6,3)	
			(1,7,93) (2,7,73) (3,7,48) (4,7,28) (5,7,15) (6,7,8) (7,7,4)	
		};
		\end{axis}
		\end{tikzpicture}	
		\qquad
		\begin{tikzpicture}[scale=0.7]
		\begin{axis}[
		xlabel={$log_{2}(nbAgents)$},
		ylabel={$log_{2}(domainSize)$},
		zlabel={$tightness$}
		]
		\addplot3[surf,mesh/rows=7] coordinates {
			(1,1,50) (2,1,0) (3,1,0) (4,1,0) (5,1,0) (6,1,0) (7,1,0)
			(1,2,50) (2,2,25) (3,2,0) (4,2,0) (5,2,0) (6,2,0) (7,2,0)
			(1,3,63) (2,3,38) (3,3,25) (4,3,13) (5,3,0) (6,3,0) (7,3,0)
			(1,4,75) (2,4,50) (3,4,31) (4,4,13) (5,4,6) (6,4,0) (7,4,0)
			(1,5,84) (2,5,59) (3,5,38) (4,5,19) (5,5,9) (6,5,3) (7,5,0)
			(1,6,89) (2,6,67) (3,6,42) (4,6,25) (5,6,13) (6,6,6) (7,6,3)
			(1,7,93) (2,7,73) (3,7,48) (4,7,28) (5,7,15) (6,7,8) (7,7,4)
		};
		\end{axis}
		
		\end{tikzpicture}
		\caption{Parametrization for a probability of $50\%$ to have a solution -- theoretical (left) and experimental (right) values} \label{figure:parametersDMS}
	\end{center}
\end{figure}
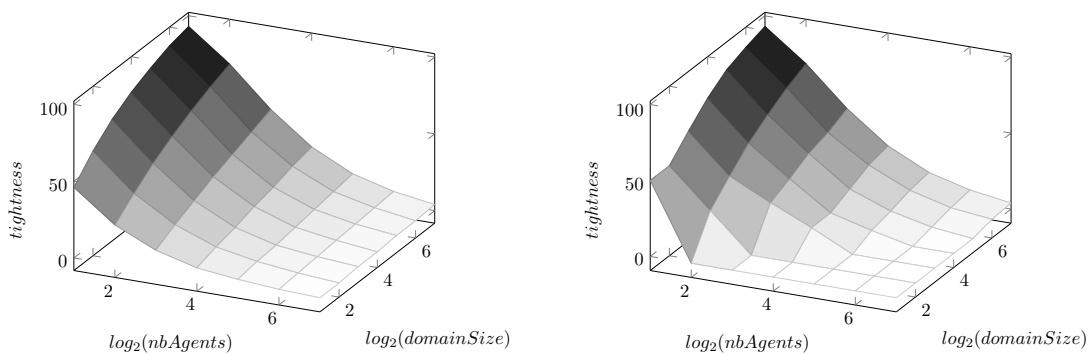
	
	For example, to have a probability of $50\%$ to have a solution, 
	we see that a DMS with $8$ agents ($x=3$), and 
	a domain size of $8$ ($y=3$), needs to be parametrized with a 
	tightness of $27\%$ ($z=27$).
	For a problem with $16$ agents  (\ie $x=4$) and a same domain 
	size ($y=3$) and a probability to have a solution, the tightness becomes $14\%$ 
	($z=14$). Note that in the experimental plot, values may slightly differ due 
	to the discrete nature of domains size.
	Later, this formula will be used to generate 
	problems with relevant parameters values and that are not under- or over-constrained. 
	We will verify that probability to find a solution to  DMS is negatively correlated 
	with the number of agents, constraints tightness, and positively correlated with 
	domains size.
		Each set of reported experiments is an average estimation of $50$ instances for 
		the different algorithms (\SyncBT{} -- \SyncBTU{}, \ABT{} -- \ABTU{},
		\ADOPT{} -- \ADOPTU{}, \DBO{} -- \DBOU{} and \DSA{} -- \DSAU{}).

	\subsection{Experiments on \UDisCP} \label{result}

	Tables~\ref{table:plotDisCSP} and~\ref{table:plotDCOP}
	show the average privacy loss per agent during the execution 
	of the solving algorithms for \DisCSPs (\SyncBT, \ABT), \DCOPs (\DBO, \DSA, and \ADOPT), 
	and of their respective  extensions, 
	with several values for the number of agents, 
	the domain size and the constraints tightness.
	Bold data points show instances with high privacy loss (from $10.0$ to $20.0$). Privacy loss of $20.0$ is the maximal value before interrupting solving.
	Empty data points show instances with low privacy loss (below $0.1$).
	Moreover, we refer to a data point in these tables as $solver$($nbAgents$, $domainSize$, $tightness$).
	For example, $\SyncBT$($10$,$20$,$30$) refers to the average privacy loss for instances 
	with $10$ agents, a domain size of $20$, and a constraint tightness of $30\%$
	solved with \SyncBT, namely $6.7$.

	\begin{table}[h]
	\centering
	\caption{Privacy Loss in \DisCSPs and \UDisCSPs	}
	\label{table:plotDisCSP}			
	\scalebox{0.7}	{
		\begin{tabular}{p{0.15\textwidth}
				p{0.0\textwidth}p{0.03\textwidth}p{0.03\textwidth}
				p{0.03\textwidth}p{0.03\textwidth}p{0.03\textwidth}
				p{0.0\textwidth}p{0.03\textwidth}p{0.03\textwidth}
				p{0.03\textwidth}p{0.03\textwidth}p{0.03\textwidth}
				p{0.0\textwidth}p{0.03\textwidth}p{0.03\textwidth}
				p{0.03\textwidth}p{0.03\textwidth}p{0.03\textwidth}	
				p{0.0\textwidth}p{0.03\textwidth}p{0.03\textwidth}
				p{0.03\textwidth}p{0.03\textwidth}p{0.03\textwidth}														
			} \\ \toprule	
			nbAgents & \multicolumn{24}{c}{$10$}   \\			
			\cmidrule{3-25}	
			domainSize & \multicolumn{6}{c}{$10$} & \multicolumn{6}{c}{$20$} & \multicolumn{6}{c}{$30$} & \multicolumn{6}{c}{$40$}\\
			\cmidrule{3-7} \cmidrule{9-13} \cmidrule{15-19} \cmidrule{21-25} 
			tightness(\%) &&
			$10$&$20$&$30$&$40$&$50$&&
			$10$&$20$&$30$&$40$&$50$&&
			$10$&$20$&$30$&$40$&$50$&&
			$10$&$20$&$30$&$40$&$50$  \\ \hline
			\SyncBT&&
			$0.7$ & $3.7$ & $3.8$ & $3.3$ & $3.5$ &     &
			$0.7$ & $4.9$ & $6.7$ & $6.0$ & $7.0$ &     &
			$0.7$ & $5.2$ & $8.8$ & $9.3$ & $\textbf{10.5}$ &     &
			$0.7$ & $5.4$ & $\textbf{10.3}$ & $\textbf{12.1}$ & $\textbf{14.0}$     \\
			\SyncBTU&&
			$0.7$ & $3.4$ & $2.8$ & $2.0$ & $0.6$ &     &
			$0.7$ & $4.5$ & $4.9$ & $3.7$ & $1.2$ &     &
			$0.7$ & $4.8$ & $6.5$ & $5.7$ & $1.8$ &     &
			$0.7$ & $4.9$ & $7.9$ & $7.3$ & $2.4$ \\
			\hdashline
			\ABT&&
			$3.1$ & $7.2$ & $\textbf{12.4}$ & $\textbf{13.9}$ & $1.6$ &     &
			$3.1$ & $9.5$ & $\textbf{20.0}$ & $\textbf{20.0}$ & $3.2$ &     &
			$3.1$ & $\textbf{10.3}$ & $\textbf{20.0}$ & $\textbf{20.0}$ & $4.8$ & &       
			$3.1$ & $\textbf{10.4}$ & $\textbf{20.0}$ & $\textbf{20.0}$ & $6.4$ \\
			\ABTU&&
			$2.5$ & $5.2$ & $6.1$ & $6.9$ & $1.5$ &     &
			$2.5$ & $6.9$ & $\textbf{10.7}$ & $\textbf{11.3}$ & $3.0$ &     &
			$2.5$ & $7.4$ & $\textbf{14.2}$ & $\textbf{19.5}$ & $4.5$ &     &
			$2.5$ & $7.5$ & $\textbf{16.6}$ & $\textbf{20.0}$ & $6.0$ \\
			\bottomrule	
			nbAgents & \multicolumn{24}{c}{$20$}   \\			
			\cmidrule{3-25}	
			domainSize & \multicolumn{6}{c}{$10$} & \multicolumn{6}{c}{$20$} & \multicolumn{6}{c}{$30$} & \multicolumn{6}{c}{$40$}\\
			\cmidrule{3-7} \cmidrule{9-13} \cmidrule{15-19} \cmidrule{21-25} 
			tightness(\%) &&
			$10$&$20$&$30$&$40$&$50$&&
			$10$&$20$&$30$&$40$&$50$&&
			$10$&$20$&$30$&$40$&$50$&&
			$10$&$20$&$30$&$40$&$50$  \\ \hline
			\SyncBT&&
			$0.5$ & $0.4$ & $0.1$ &     &     &     &
			$0.7$ & $1.1$ & $0.1$ &     &     &     &
			$0.7$ & $4.7$ & $0.1$ &     &     &     &
			$0.7$ & $1.9$ & $0.1$ &     &     \\
			\SyncBTU&&
			$0.5$ & $0.3$ & $0.1$ &     &     &     &
			$0.7$ & $1.0$ & $0.1$ &     &     &     &
			$0.7$ & $4.3$ & $0.1$ &     &     &     &
			$0.7$ & $1.8$ & $0.1$ &     &     \\
			\hdashline
			\ABT&&
			$2.3$ & $0.7$ & $0.1$ &     &     &     &
			$2.9$ & $2.1$ & $0.1$ &     &     &     &
			$3.0$ & $9.3$ & $0.1$ &     &     &     &
			$3.1$ & $3.8$ & $0.1$ &     &     \\
			\ABTU&&
			$1.8$ & $0.5$ & $0.1$ &     &     &     &
			$2.3$ & $1.5$ & $0.1$ &     &     &     &
			$2.4$ & $6.7$ & $0.1$ &     &     &     &
			$2.5$ & $2.8$ & $0.1$ &     &     \\
			\bottomrule	
			nbAgents & \multicolumn{24}{c}{$40$}   \\			
			\cmidrule{3-25}	
			domainSize & \multicolumn{6}{c}{$10$} & \multicolumn{6}{c}{$20$} & \multicolumn{6}{c}{$30$} & \multicolumn{6}{c}{$40$}\\
			\cmidrule{3-7} \cmidrule{9-13} \cmidrule{15-19} \cmidrule{21-25} 
			tightness(\%) &&
			$10$&$20$&$30$&$40$&$50$&&
			$10$&$20$&$30$&$40$&$50$&&
			$10$&$20$&$30$&$40$&$50$&&
			$10$&$20$&$30$&$40$&$50$  \\ \hline
			\SyncBT&&
			$0.1$ &     &   &   &   &   &
			$0.2$ &     &   &   &   &   &
			$0.3$ &     &   &   &   &   &
			$0.3$ & $0.1$ &   &   &  \\
			\SyncBTU&&
			$0.1$ &     &   &   &   &   &
			$0.2$ &     &   &   &   &   &
			$0.3$ &     &   &   &   &   &
			$0.3$ & $0.1$ &   &   &  \\
			\hdashline
			\ABT&&
			$0.4$ &     &   &   &   &   &
			$0.8$ &     &   &   &   &   &
			$0.9$ &     &   &   &   &   &
			$1.1$ & $0.1$ &   &   &  \\
			\ABTU&&
			$0.4$ &     &   &   &   &   &
			$0.7$ &     &   &   &   &   &
			$0.9$ &     &   &   &   &   &
			$1.1$ & $0.1$ &   &   &  \\ 
			\bottomrule	
	\end{tabular}}
\end{table}			

\begin{table}[h]
	\centering
	\caption{Privacy Loss in \DCOPs and \UDCOPs	}
	\label{table:plotDCOP}			
	\scalebox{0.7}	{
		\begin{tabular}{p{0.15\textwidth}
				p{0.0\textwidth}p{0.03\textwidth}p{0.03\textwidth}
				p{0.03\textwidth}p{0.03\textwidth}p{0.03\textwidth}
				p{0.0\textwidth}p{0.03\textwidth}p{0.03\textwidth}
				p{0.03\textwidth}p{0.03\textwidth}p{0.03\textwidth}
				p{0.0\textwidth}p{0.03\textwidth}p{0.03\textwidth}
				p{0.03\textwidth}p{0.03\textwidth}p{0.03\textwidth}	
				p{0.0\textwidth}p{0.03\textwidth}p{0.03\textwidth}
				p{0.03\textwidth}p{0.03\textwidth}p{0.03\textwidth}														
			} \\ \toprule	
			nbAgents & \multicolumn{24}{c}{10}   \\			
			\cmidrule{3-25}	
			domainSize & \multicolumn{6}{c}{10} & \multicolumn{6}{c}{20} & \multicolumn{6}{c}{30} & \multicolumn{6}{c}{40}\\
			\cmidrule{3-7} \cmidrule{9-13} \cmidrule{15-19} \cmidrule{21-25} 
			tightness(\%) &&
			$10$&$20$&$30$&$40$&$50$&&
			$10$&$20$&$30$&$40$&$50$&&
			$10$&$20$&$30$&$40$&$50$&&
			$10$&$20$&$30$&$40$&$50$  \\ \hline
			\DBO&&
			$0.3$ & $0.6$ & $1.1$ & $1.1$ & $1.6$ &   &
			$0.4$ & $0.8$ & $1.4$ & $1.4$ & $2.1$ &   &
			$0.6$ & $1.1$ & $2.0$ & $2.0$ & $2.9$ &   &
			$0.8$ & $1.5$ & $2.8$ & $2.8$ & $4.1$ \\
			\DBOU&&
			$0.2$ & $0.4$ & $0.4$ & $0.8$ & $1.0$ &   &
			$0.3$ & $0.5$ & $0.5$ & $1.0$ & $1.3$ &   &
			$0.4$ & $0.7$ & $0.7$ & $1.4$ & $1.8$ &   &
			$0.5$ & $1.0$ & $1.0$ & $2.0$ & $2.6$ \\
			\hdashline
			\DSA&&
			$1.4$ & $3.6$ & $5.0$ & $6.4$ & $8.4$ &   &
			$1.8$ & $4.7$ & $6.5$ & $8.3$ & $\textbf{10.9}$ &   &
			$2.9$ & $6.6$ & $9.1$ & $\textbf{11.7}$ & $\textbf{15.3}$ &   &
			$3.6$ & $9.2$ & $\textbf{12.8}$ & $\textbf{16.3}$ & $\textbf{20.0}$ \\
			\DSAU&&
			$1.4$ & $2.0$ & $2.2$ & $2.6$ & $3.0$ &   &
			$1.8$ & $2.6$ & $2.9$ & $3.4$ & $3.9$ &   &
			$2.6$ & $3.6$ & $4.0$ & $4.7$ & $5.5$ &   &
			$3.6$ & $5.1$ & $5.6$ & $6.6$ & $7.6$ \\
			\hdashline
			\ADOPT&&
			$3.6$ & $4.2$ & $5.6$ & $6.8$ & $8.8$ &   &
			$4.4$ & $5.5$ & $7.3$ & $8.8$ & $\textbf{11.4}$ &   &
			$6.2$ & $7.6$ & $\textbf{10.2}$ & $\textbf{12.4}$ & $\textbf{16.0}$ &   &
			$8.7$ & $1.7$ & $\textbf{14.2}$ & $\textbf{17.3}$ & $\textbf{20.0}$ \\
			\ADOPTU&&
			$2.4$ & $3.0$ & $3.4$ & $4.0$ & $4.2$ &   &
			$3.1$ & $3.9$ & $4.4$ & $5.2$ & $5.5$ &   &
			$4.4$ & $5.5$ & $6.2$ & $7.3$ & $7.6$ &   &
			$6.2$ & $7.6$ & $8.7$ & $\textbf{10.2}$ & $\textbf{10.7}$ \\
			\bottomrule	
			nbAgents & \multicolumn{24}{c}{$20$}   \\			
			\cmidrule{3-25}	
			domainSize & \multicolumn{6}{c}{$10$} & \multicolumn{6}{c}{20} & \multicolumn{6}{c}{$30$} & \multicolumn{6}{c}{$40$}\\
			\cmidrule{3-7} \cmidrule{9-13} \cmidrule{15-19} \cmidrule{21-25} 
			tightness(\%) &&
			$10$&$20$&$30$&$40$&$50$&&
			$10$&$20$&$30$&$40$&$50$&&
			$10$&$20$&$30$&$40$&$50$&&
			$10$&$20$&$30$&$40$&$50$  \\ \hline
			\DBO&&
			$0.2$ & $0.1$ & $0.1$ &   &   &   &
			$0.4$ & $0.7$ & $0.1$ &   &   &   &
			$0.5$ & $1.0$ & $0.1$ &   &   &   &
			$0.7$ & $1.4$ & $0.1$ &   &  \\ 
			\DBOU&&
			$0.2$ & $0.1$ & $0.1$ &   &   &   &
			$0.2$ & $0.1$ & $0.1$ &   &   &   &
			$0.3$ & $0.7$ & $0.1$ &   &   &   &
			$0.5$ & $0.4$ & $0.1$ &   &  \\ 
			\hdashline
			\DSA&&
			$1.0$ & $0.4$ & $0.1$ &   &   &   &
			$1.6$ & $1.0$ & $0.1$ &   &   &   &
			$2.5$ & $5.9$ & $0.1$ &   &   &   &
			$3.2$ & $3.3$ & $0.1$ &   &  \\ 
			\DSAU&&
			$1.0$ & $0.2$ & $0.1$ &   &   &   &
			$1.6$ & $0.6$ & $0.1$ &   &   &   &
			$2.5$ & $3.3$ & $0.1$ &   &   &   &
			$3.4$ & $1.8$ & $0.1$ &   &  \\ 
			\hdashline
			\ADOPT&&
			$2.5$ & $0.5$ & $0.1$ &   &   &   &
			$4.0$ & $1.2$ & $0.1$ &   &   &   &
			$6.0$ & $6.9$ & $0.1$ &   &   &   &
			$8.3$ & $0.6$ & $0.1$ &   &  \\ 
			\ADOPTU&&
			$1.8$ & $0.3$ & $0.1$ &   &   &   &
			$2.8$ & $0.2$ & $0.1$ &   &   &   &
			$4.2$ & $4.9$ & $0.1$ &   &   &   &
			$5.9$ & $2.8$ & $0.1$ &   &  \\ 
			\bottomrule	
			nbAgents & \multicolumn{24}{c}{$40$}   \\			
			\cmidrule{3-25}	
			domainSize & \multicolumn{6}{c}{$10$} & \multicolumn{6}{c}{$20$} & \multicolumn{6}{c}{$30$} & \multicolumn{6}{c}{$40$}\\
			\cmidrule{3-7} \cmidrule{9-13} \cmidrule{15-19} \cmidrule{21-25} 
			tightness(\%) &&
			$10$&$20$&$30$&$40$&$50$&&
			$10$&$20$&$30$&$40$&$50$&&
			$10$&$20$&$30$&$40$&$50$&&
			$10$&$20$&$30$&$40$&$50$  \\ \hline
			\DBO&&
			$0.1$ &     &   &   &   &   &
			$0.1$ &     &   &   &   &   &
			$0.6$ &     &   &   &   &   &
			$0.8$ & $0.1$ &   &   &  \\
			\DBOU&&
			$0.1$ &     &   &   &   &   &
			$0.1$ &     &   &   &   &   &
			$0.1$ &     &   &   &   &   &
			$0.2$ & $0.1$ &   &   &  \\
			\hdashline
			\DSA&&
			$0.3$ &     &   &   &   &   &
			$0.5$ &     &   &   &   &   &
			$1.0$ &     &   &   &   &   &
			$1.3$ & $0.1$ &   &   &  \\
			\DSAU&&
			$0.3$ &     &   &   &   &   &
			$0.4$ &     &   &   &   &   &
			$1.0$ &     &   &   &   &   &
			$1.3$ & $0.1$ &   &   &  \\
			\hdashline
			\ADOPT&&
			$0.8$ &     &   &   &   &   &
			$1.2$ &     &   &   &   &   &
			$2.5$ &     &   &   &   &   &
			$3.0$ & $0.1$ &   &   &  \\
			\ADOPTU&&
			$0.5$ &     &   &   &   &   &
			$0.8$ &     &   &   &   &   &
			$1.8$ &     &   &   &   &   &
			$2.2$ & $0.1$ &   &   &  \\ 
			\bottomrule	
	\end{tabular}}
\end{table}	
	
	\paragraph{}Generally, we see that for both Tables~\ref{table:plotDisCSP} and~\ref{table:plotDCOP},
	increasing the number of agents implies a reduction of the number of solutions. 
	Thus, problems being over-constrained, agents interrupt the solving faster, which 
	can explain the reduction of average privacy loss for problems with many agents. 
	Indeed, we see that most instances with $10$ agents have a significant average 
	privacy loss, while instances with $40$ agents have low privacy loss.
	For example, in Table~\ref{table:plotDisCSP}, for $solver$($10$, $10$, $50$), all 
	solvers imply a privacy loss ($3.5$, $0.6$, $1.6$, $1.5$) for \SyncBT{} -- \SyncBTU and \ABT{} -- \ABTU, respectively. However, for $solver$($40$, $10$, $50$) all the previous solvers have empty data points. Similarly for Table~\ref{table:plotDCOP}, results obtained for $40$ agents give also a lower privacy losses for the same reasons.
	
	\paragraph{} Moreover, we also see that for these two tables,
	high privacy loss is correlated with a high tightness and a high domain size. Table~\ref{table:plotDisCSP} shows that the number of  high privacy loss  depend on the tightness (more than $30\%$) and a domain size up to $30$. 	For example, for $ABTU$($10$,$10$,$50$), privacy loss is $1.5$, while it is 
		$6.0$ for $\ABTU$($10$,$40$,$50$).  
	Table~\ref{table:plotDCOP} shows that high privacy loss (bold data points) never occurs for instances with a domain size of $10$, while it does for most problems with a domain size of $40$. 	For example, for $\ADOPT$($10$,$10$,$50$), privacy loss is $8.8$, while  $\ADOPT$($10$,$40$,$50$) gives 	$20.0$. For these two tables, the main reason is that a higher domain size 
	allows more solution proposals and avoids premature interruptions.

 \paragraph{} Finally, recall  that the difference of values between different families of solvers is explained by the different types of privacy considered (\ie domain, assignment and constraint), as 
 the number of revelations of assignments differs from the number of revelations of constraints.   
	Table~\ref{table:plotDisCSP} measures domain privacy loss for agents
	during the execution of \SyncBT, \ABT, and their extensions.
	\SyncBT and \SyncBTU are better than \ABT and \ABTU at preserving privacy,
	likely due to the increase of exchanged messages  with asynchronous solvers, as agents 
	can run concurrently.	
	 Table~\ref{table:plotDCOP} measures assignment privacy with \ADOPT and \ADOPTU.  For $\ADOPT$($10$,$40$,$t$) with $t \in \{10,20, 30, 40, 50\}$, the  assignment privacy loss is $\{6.2,7.6, 10.2, 12.4, 16\}$. However with the same parameters for \ADOPTU, it is  $\{4.4,5.5, 6.2, 7.3, 7.6\}$. The loss is reduced with our extension as search is driven by objective to minimize the assignment privacy loss, and interrupt solving if  necessary. 
	Table~\ref{table:plotDCOP} measures constraint privacy loss with \DBO{} -- \DBOU and 
\DSA{} -- \DSAU. We see that \DBO and \DBOU are better than
\DSA and \DSAU, likely due to the fact that with \DBO, only one agent changes its value at each iteration, while with \DSA all agents do, which implies an increase of the number 
of exchanged messages.
	For example, in Table~\ref{table:plotDCOP} for $solver$($10$, $20$, $50$) 
	privacy loss drops from $2.1$ to $1.3$ and from $10.9$ to $3.9$  for 
	\DBO{} -- \DBOU and \DSA{} -- \DSAU respectively.	
For all these algorithms dealing with different types of privacy, an increase of  privacy loss with initial solvers implies a better preservation with extended ones.

	\section{Conclusion}\label{Conclusions}
	
	While many approaches have been proposed recently for dealing with privacy in distributed constraint satisfaction and optimization problems, 
	none of them is exempt from limitations. These approaches may require particular properties from the initial problem, or may consider certain aspects of privacy only. In this work, we propose the Utilitarian Distributed Constrained Problem  
	(\UDisCP{}) framework. 
	The framework models the privacy loss for the revelation of an 
	agent's constraints as a utility function, letting agents integrate privacy requirements directly in their search process. Solving the problem then consists
	in finding the best compromise between 
	solution quality and privacy loss, instead of focusing only on solution quality.
	We propose extensions to existing algorithms for \DisCSPs{} (\SyncBTU{}, \ABTU{}) and \DCOPs{} 
	(\DBOU{}, \DSAU{}, and \ADOPTU{}) 
	that let agents use information about 
	privacy to modify their behaviour and guide their search process, by proposing values that reduce the amount of privacy loss, and 
	compare them on different types of distributed meeting scheduling problems. 
	The comparison shows that explicit modelling  and reasoning with the utility of privacy allows for significant savings in privacy with minimal 
	impact on the quality of the achieved solutions. 
	
	Our approach has several possible extensions. For future works, we first plan to
	extend our models and algorithms to more general problems modelled with constraints, 
	namely problems with n-ary constraints (not only binary ones) as well as 
	multi-variable problems, where each agent may control several variables. 
	In this case, 
	the solution of an agent is a set of assignments, rather than a single assignment.
	Further, we want to investigate the notion of ethics~\cite{cointe2016}, implying that agents may have remorse for lying when modifying their behaviour for privacy. These notion of ethics leads to the building of communities of interests between agents~\cite{zardi2014,zardi2016}, where privacy costs will differ according to the recipient of a message, instead of being random. 
	A \DisCSP modelling was used for road traffic in~\cite{Doniec2008,DoniecMPE08},
	and we would like to apply our model of privacy for road traffic simulation where the agents/drivers may choose whether to reveal some private data concerning for example their driving, or their habits. Indeed, we think that our privacy model may also be adapted to other applications such as  multi-robot exploration~\cite{Monier2011} or smart energy~\cite{Rust2016}. 
	
	\section*{Acknowledgements}
	Authors gratefully acknowledge the reviewers for their constructive remarks, 
	allowing to improve our article. 
	This work was realized while the first author was a visiting researcher 
	at Florida Institute of Technology (Sept.~$2015$ -- May~$2016$). 

	\bibliographystyle{plain}	
	\bibliography{arxiv}

\end{document}